\documentclass[sigconf, nonacm]{acmart}

\usepackage{array}
\usepackage{tabularx}
\usepackage{listings}
\usepackage[most]{tcolorbox}
\usepackage[table]{xcolor}
\usepackage{placeins}
\usepackage{multirow}
\usepackage{makecell}
\usepackage{xurl}
\usepackage{balance}

\lstdefinestyle{prompttxt}{
  basicstyle=\ttfamily\small,
  columns=fullflexible,
  breaklines=true,
  breakatwhitespace=false,
  keepspaces=true,
  showstringspaces=false,
  tabsize=2
}

\definecolor{lightgray}{gray}{0.93}

\newtcolorbox{promptbox}[1]{
  enhanced,
  breakable,
  colback=gray!1,
  colframe=black!40,
  boxrule=0.6pt,
  arc=2mm,
  outer arc=2mm,
  left=6pt, right=6pt, top=6pt, bottom=6pt,
  title={#1},
  coltitle=white,
  fonttitle=\bfseries,
  colbacktitle=black!80,
  boxed title style={arc=1mm, outer arc=1mm},
  attach boxed title to top left={yshift=-2mm, xshift=2mm},
  drop shadow
}

\newcommand{\promptboxfromfile}[2]{%
  \begin{promptbox}{#1}
    \lstinputlisting[style=prompttxt]{#2}
  \end{promptbox}%
}

\AtBeginDocument{%
  }

\begin{document}

\title{AlphaOPT: Formulating Optimization Programs with Self-Improving LLM Experience Library}

\author{Minwei Kong}
\authornote{Equal contribution.}
\affiliation{%
  \institution{Singapore-MIT Alliance for Research and Technology}
  \country{Singapore}
}
\email{minwei.kong@smart.mit.edu}

\author{Ao Qu}
\authornotemark[1]
\authornote{Corresponding author.}
\affiliation{%
  \institution{Massachusetts Institute of Technology}
  \country{Cambridge, MA, United States}
}
\email{qua@mit.edu}

\author{Xiaotong Guo}
\affiliation{%
  \institution{Massachusetts Institute of Technology}
  \country{Cambridge, MA, United States}
}
\email{guoxiaotong95@gmail.com}

\author{Wenbin Ouyang}
\affiliation{%
  \institution{Massachusetts Institute of Technology}
  \country{Cambridge, MA, United States}
}
\email{oywenbin@mit.edu}

\author{Chonghe Jiang}
\affiliation{%
  \institution{Massachusetts Institute of Technology}
  \country{Cambridge, MA, United States}
}
\email{chonghej@mit.edu}

\author{Han Zheng}
\affiliation{%
  \institution{Massachusetts Institute of Technology}
  \country{Cambridge, MA, United States}
}
\email{hanzheng@mit.edu}

\author{Yining Ma}
\email{yiningma@mit.edu}
\affiliation{%
  \institution{Massachusetts Institute of Technology}
  \country{Cambridge, MA, United States}
}

\author{Dingyi Zhuang}
\affiliation{%
  \institution{Massachusetts Institute of Technology}
  \country{Cambridge, MA, United States}
}
\email{dingyi@mit.edu}

\author{Yuhan Tang}
\affiliation{%
  \institution{Massachusetts Institute of Technology}
  \country{Cambridge, MA, United States}
}
\email{yhtang@mit.edu}

\author{Junyi Li}
\affiliation{%
  \institution{Singapore-MIT Alliance for Research and Technology}
  \country{Singapore}
}
\email{junyi.li@smart.mit.edu}

\author{Shenhao Wang}
\affiliation{%
  \institution{University of Florida}
  \country{Gainesville, FL, USA}
}
\email{shenhaowang@ufl.edu}

\author{Haris Koutsopoulos}
\affiliation{%
  \institution{Northeastern University}
  \country{Boston, MA, USA}
}
\email{h.koutsopoulos@northeastern.edu}

\author{Hai Wang}
\affiliation{%
  \institution{Singapore Management University}
  \country{Singapore}
}
\email{haiwang@smu.edu.sg}

\author{Cathy Wu}
\affiliation{%
  \institution{Massachusetts Institute of Technology}
  \country{Cambridge, MA, United States}
}
\email{cathywu@mit.edu}

\author{Jinhua Zhao}
\affiliation{%
  \institution{Massachusetts Institute of Technology}
  \country{Cambridge, MA, United States}
}
\email{jinhua@mit.edu}

\renewcommand{\shortauthors}{Kong et al.}

\begin{abstract}
Optimization modeling underlies critical decision-making across industries, yet remains difficult to automate: natural-language problem descriptions must be translated into precise mathematical formulations and executable solver code. Existing LLM-based approaches typically rely on brittle prompting or costly retraining, both of which offer limited generalization. Recent work suggests that large models can improve via experience reuse, but how to systematically acquire, refine, and reuse such experience in structurally constrained settings remains unclear. We present \textbf{AlphaOPT}, a self-improving experience library that enables LLMs to learn optimization modeling knowledge from limited supervision, including answer-only feedback without gold-standard programs, annotated reasoning traces, or parameter updates. AlphaOPT operates in a continual two-phase cycle: a \emph{Library Learning} phase that extracts solver-verified, structured insights from failed attempts, and a \emph{Library Evolution} phase that refines the applicability of stored insights based on aggregate evidence across tasks. This design allows the model to accumulate reusable modeling principles, improve transfer across problem instances, and maintain bounded library growth over time. Evaluated on multiple optimization benchmarks, AlphaOPT steadily improves as more training data become available (65\% $\rightarrow$ 72\% from 100 to 300 training items) and outperforms the strongest baseline by 9.1\% and 8.2\% on two out-of-distribution datasets. These results demonstrate that structured experience learning, grounded in solver feedback, provides a practical alternative to retraining for complex reasoning tasks requiring precise formulation and execution. 
AlphaOPT code and data are available at \href{https://github.com/Minw913/AlphaOPT}{https://github.com/Minw913/AlphaOPT}.
\end{abstract}



\maketitle

\section{Introduction}

Optimization models support critical decision-making in finance, manufacturing, marketing, transportation, and logistics \citep{ahmaditeshnizi2024optimus,bertsimas1997introduction,NL4OPT}. Beyond improving efficiency, automating the optimization workflow lowers the barrier to operations research expertise in industry and enables non-experts to prototype faster, iterate on formulations, and deploy solver-backed decisions at scale. Yet this process has long been challenging, because informal and often ambiguous specifications must be mapped to precise, domain-specific formulations and paired with appropriate code and solvers. This creates major bottlenecks for end-to-end automation \citep{JiangShu2025llmopt}.

Advances in large language models (LLMs) make this vision increasingly feasible: They can parse natural language requirements \citep{ouyang2022training}, generate executable programs \citep{nijkampcodegen,jimenez2024swebench}, and orchestrate downstream tools \citep{DBLP:conf/iclr/QinLYZYLLCTQZHT24}. Two main lines of work have emerged. Prompt-based systems steer general LLMs with structured prompts and tool use \citep{xiao2023chain, thind2025optimai, ahmaditeshnizi2024optimus, zhang2025or}. Fine-tuning approaches adapt models on domain corpora and benchmarks \citep{huang2025orlm, yang2024optibench}. Despite this progress, both families face important limitations. Prompt-based systems improve only within the scope of their fixed prompt templates and are sensitive to small wording changes and domain shifts. Fine-tuned models, on the other hand, require retraining as new data arrives and, critically, most benchmarks and datasets in the community (e.g., NLP4LP \citep{ahmaditeshnizi2024optimus}, MAMO \citep{huang2024mamo}, IndustryOR \citep{huang2025orlm}) provide only final programs or solutions rather than the intermediate reasoning that governs modeling decisions, which limits the generalizability of fine-tuning approaches.

This motivates a new learning paradigm for optimization formulation: Instead of relying solely on prompts or retraining, LLMs should continually improve by accumulating, refining, and reusing solver-verified modeling insights. Indeed, recent approaches such as Reflexion~\citep{shinn2023reflexion}, STaR~\citep{zelikman2022star}, ExpeL~\citep{zhao2023expel}, and AlphaEvolve~\citep{novikov2025alphaevolve} demonstrate that large models can improve through experiential reuse, storing reflections, rationales, or code edits and applying them in new tasks. These methods have been effective in common reasoning tasks such as information retrieval~\citep{yang-etal-2018-hotpotqa} and web browsing~\citep{yao2022webshop}, but they face limitations for optimization problems. First, their experiences are largely preserved as free-form text or edits, without explicit applicability semantics or refinement mechanisms to represent, verify, and refine applicability-sensitive modeling knowledge; in optimization tasks, applying mismatched experiences can therefore have detrimental effects. Second, when generating experiences, these methods often lack rigorous replay-based validation, or rely on verifications limited to task-level outcomes (e.g., rewards, final answers, or test-case success), which does not guarantee that the underlying knowledge is structurally valid or transferable. Third, their experiences tend to be organized as a flat list of items, whereas operations research involves diverse insights across different modeling stages and problem types, inherently requiring a more efficient organization of acquired knowledge.

Regarding these challenges, we introduce \textbf{AlphaOPT}, an experience learning framework for optimization formulation that builds and continually refines a library of solver-verified modeling insights. In contrast to prior approaches that store experience as unstructured text or static edits, AlphaOPT formulates experience as structured, reusable insights with explicit applicability semantics. Each insight is defined as a four-tuple(\emph{taxonomy}, \emph{condition(applicability)}, \emph{explanation}, \emph{example}), specifying not only what modeling rule to reuse, but when and why it should be applied given a natural-language problem description. AlphaOPT improves performance through a continual two-phase cycle without updating model parameters. In the \emph{Library Learning} phase, the framework self-explores optimization tasks and autonomously constructs a hierarchical taxonomy to organize insights across modeling stages, while relying on solver verification to ensure correctness even in the absence of gold-standard programs. In the \emph{Library Evolution} phase, AlphaOPT diagnoses systematic misalignments between tasks and insights and refines applicability conditions based on aggregate evidence, preventing insights from being either overly narrow or overly general. Compared to existing experience learning methods, AlphaOPT introduces \textbf{three key advances}: (1) a self-updating hierarchical taxonomy for organizing insights; (2) solver-verified insight generation and updating; and (3) an aggregate-driven refinement mechanism that continuously improves the precision and generalizability of insight applicability.

We conduct quantitative experiments across multiple benchmarks and baselines. Compared with other approaches, the results show that AlphaOPT not only achieves state-of-the-art performance on multiple benchmark datasets without any gold-standard formulation or program, but also demonstrates strong generalization on out-of-distribution datasets, and continual performance growth with converging library size. These results demonstrate the efficacy of experience learning for optimization formulation, which paves the way toward more challenging settings, such as efficient program formulation and large-scale optimization. Compared to existing approaches, AlphaOPT makes the following contributions:
\begin{itemize}
\item We introduce the first experience learning framework for natural-language optimization modeling, which represents solver-verified modeling knowledge as condition-aware insights that are continually refined over time.
\item Our experiments show that AlphaOPT generalizes strongly beyond training distributions, achieving state-of-the-art performance on the LogiOR and OptiBench benchmarks.
\item We provide extensive qualitative and quantitative analyses of the learned experience library and its refinement dynamics, demonstrating explicit, interpretable, and human-inspectable knowledge transfer in operations research domains.
\end{itemize}

\begin{figure*}[t]
  \centering
  \includegraphics[width=\linewidth]{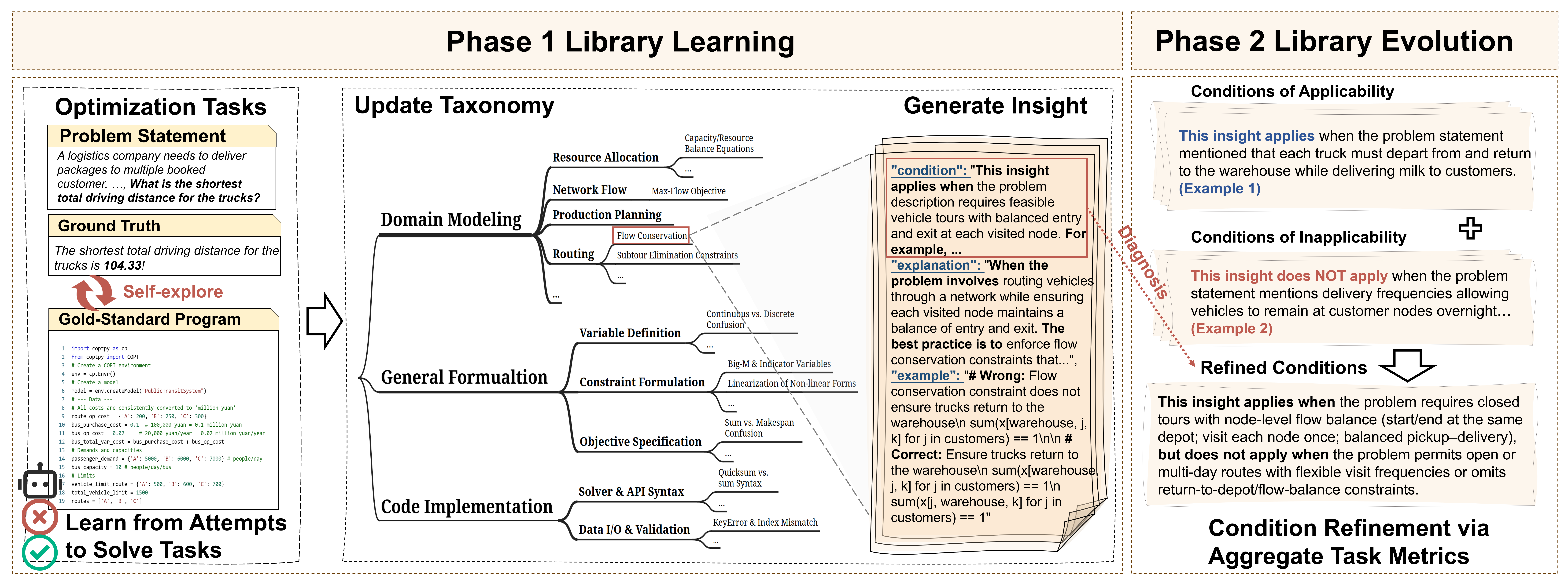}
  \caption{AlphaOPT iteratively learns an experience library through two phases: a Library Learning phase that extracts structured insights from failed attempts and organizes them in a hierarchical taxonomy, and a Library Evolution phase that refines insight applicability using evidence aggregated across all associated tasks to avoid over- or under-generalization.}
  \Description{The figure shows a two-phase process for learning and refining optimization knowledge. On the left, Phase 1 illustrates library learning. The process begins with an optimization task, including a natural language problem statement, a known correct answer, and a gold-standard mathematical program. The system attempts to solve the task, learns from both mistakes and successful solutions, and uses these attempts to extract reusable experience. This experience is organized into a structured taxonomy that covers domain modeling, general mathematical formulation, and code implementation. From this structure, the system generates concrete insights, each with a clear condition describing when it applies, an explanation of the modeling principle, and examples that contrast incorrect and correct formulations. On the right, Phase 2 illustrates library evolution. Previously generated insights are examined by comparing cases where they apply and where they do not. Using performance and error signals aggregated across tasks, the system refines the conditions under which each insight should be used. The result is a more precise description of applicability, for example distinguishing problems that require closed vehicle tours with flow balance from problems that allow open or multi-day routes. Overall, the figure shows how raw optimization attempts are transformed into structured, conditional knowledge and how that knowledge is iteratively refined to improve preciseness of the applicability and transfer across diverse problems.}
  \label{fig:framework}
\end{figure*}

\section{Related Work}
\subsection{LLMs for Solving Optimization Problems}
Recent work increasingly uses LLMs as autonomous agents for optimization, mainly following two paradigms: formulation and heuristic generation \citep{ye2024reevo, yang2025heuragenix, wu2025efficientheuristics}. 
LLM-driven formulation tasks aim to translate natural-language problem statements into solver-ready mathematical programs and executable solving code. 
Multiple benchmark datasets have been proposed to cover diverse problem families such as LP, MILP, and NLP \citep{xiao2023chain, ahmaditeshnizi2024optimus, yang2024optibench, yang2025orthought}. 

Existing approaches can be broadly grouped into prompt-based and training-based methods. Prompt-based methods use multi-step prompting and structured reasoning to guide LLMs in parsing problems, extracting variables/constraints, and synthesizing models. Representative examples include Chain-of-Experts \citep{xiao2023chain}, which leverages multi-role decomposition, and OptiMUS \citep{ahmaditeshnizi2024optimus}, which employs a self-reflect agentic pipeline with retrieval-augmented prompting to produce mathematical formulations. LEAN-LLM-OPT \citep{liang2026leanllmopt} further integrates a three-agent workflow with RAG and few-shot learning for problem typing, example retrieval, and data grounding. However, these methods lack sustained knowledge injection and long-term memory, which can lead to performance plateaus under distribution shift. Training-based methods improve robustness by building instruction corpora and fine-tuning open-source models. For example, ORLM \citep{huang2025orlm} synthesizes natural-language optimization problem data used for fine-tuning, and LLMOPT \citep{JiangShu2025llmopt} adopts a five-element schema with multi-instruction SFT on expert-labeled and augmented data. Recent post-training work further applies RL-style refinement, such as OR-R1 \citep{ding2025orr1}, to improve single-shot performance and robustness. Nonetheless, these pipelines rely heavily on high-quality, diverse data, and are often hard to interpret or audit.

Overall, despite progress in prompting, data synthesis, fine-tuning, and post-training, existing methods remain insufficient for reliable deployment and sustained improvement in real-world operations research systems.
\subsection{Decision-making Tasks with Experience Learning}
Experience Learning studies how LLM-based agents extract and reuse transferable structured experience from historical interactions to continuously improve downstream decision-making and execution \citep{zhao2023expel, shinn2023reflexion}. Such experience can take the form of episodic traces \citep{park2023generative} (e.g., tool/action sequences), semantic insights abstracted from successful or failed attempts \citep{zhao2023expel}. This line of work is primarily evaluated on long-horizon, interactive, feedback-rich decision-making tasks, including multi-tool planning and web interaction \citep{qin2023toolllm, zhou2023webarena},  repository-level program repair/code generation \citep{jimenez2023swebench}, and multi-step reasoning problem solving \citep{shinn2023reflexion, wang2024awm}.

\citet{shinn2023reflexion} propose Reflexion, which enables language agents to learn from trial-and-error by converting task feedback into verbal self-reflections that are stored as episodic memory and reused in subsequent attempts. ExpeL \citep{zhao2023expel} jointly extracts similar exemplars/trajectories together with their associated insights to guide task execution and further employs a voting-based mechanism to downweight mutually contradictory insights, thereby improving single-shot performance and cross-task generalization.
At a more symbolic end, \citet{zhu2023llmrules} explicitly induces a rule library and maintains it using quantitative filters such as coverage and confidence, enhancing controllability and robustness of experience reuse. AgentKB \citep{tang2025agentkb} and ReasoningBank \citep{ ouyang2025reasoningbank} move toward organizing an evolving reasoning memory to support self-improving agents across tasks.

However, current frameworks face systematic limitations in three key aspects. First, regarding experience applicability, most methods still rely on embedding-based similarity retrieval or coarse metadata to match experiences \citep{zhao2023expel, park2023generative, wang2024awm}, while lacking explicit, verifiable preconditions (when an experience applies, when it should not, and what assumptions it requires), which increases the risk of negative transfer under distribution shift. Second, for strict verification of experience generation, merging, and editing, existing mechanisms often depend on LLM self-assessment \citep{zhao2023expel, shinn2023reflexion, feng2025agentrr}; for instance, ExpeL's voting-based filtering can remove contradictions but does not guarantee semantic correctness. Finally, regarding cross-task generalization, existing experience libraries typically support experience accumulation and local pruning \citep{zhao2023expel, tang2025agentkb, ouyang2025reasoningbank}, but lack a systematic mechanism to assess whether an experience that helps on local source tasks may degrade performance on broader related task sets.

\section{Methodology}
\label{sec:method}
Building reliable LLM-based systems for solving optimization problems remains challenging. To begin with, gold-standard formulations and programs are scarce and may contain annotation errors~\citep{JiangShu2025llmopt, yang2025orthought}, while datasets with only answer-level supervision are underexploited~\citep{huang2024mamo,huang2025orlm,lu2025optmath}. Moreover, although optimization tasks arrive with diverse natural-language descriptions, they often share recurring modeling rules that activate under identifiable conditions; however, existing methods still struggle to learn such when-to-apply-what knowledge. Fine-tuned models~\citep{huang2025orlm, JiangShu2025llmopt} often overfit to syntax without mastering applicability, whereas prompt-based agent systems~\citep{ahmaditeshnizi2024optimus, xiao2023chain, yang2025orthought} rely on human empirical curation and lack the capacity to adapt or continually learn from larger datasets.

To this end, we propose \textbf{AlphaOPT}, which constructs an experience library of reusable insights with explicit applicability conditions, and then iteratively evolves these conditions at population level to improve generalization. As illustrated in Figure \ref{fig:framework}, it consists of two complementary phases that form a continual cycle of acquisition and refinement. The first phase, Library Learning, extracts structured insights from individual failed tasks under answer-only supervision while minimizing redundancy. The second phase, Library Evolution, diagnoses misalignments between insights and tasks and refines applicability conditions to enhance generalization while reducing confusion caused by overgeneralization. In Section \ref{sec:method:theory}, we provide a mathematical interpretation that frames library construction as maximizing task success with a size regularizer, with a convergence analysis provided in Appendix~\ref{sec:appendix:proof}.

To the best of our knowledge, our framework is the first to adapt experience learning to operations research (OR) with three key innovations: (1) structured knowledge for interpretability and auditability: Each insight is represented with taxonomy, condition, explanation, and example, which renders its applicability explicit, reviewable, and even revisable in practice; (2) solver-guided verifiability: If a program generated by LLM achieves the optimal objective under the solver, it is highly likely to be valid and can serve as a reliable anchor for extracting insights, which broadens the sources of experience collection. Newly generated insights are retested to ensure they are retrievable in the library and to succeed when replayed on its source task, which ensures that they are valid before integration; (3) refinement of experience applicability for generalizability and preciseness: Applicability conditions are refined using cross-task evidence, so insights neither over-generalize nor remain too narrow, which improves safe transfer across problem families.

\begin{figure*}[t]
  \centering
  \includegraphics[width=\textwidth]{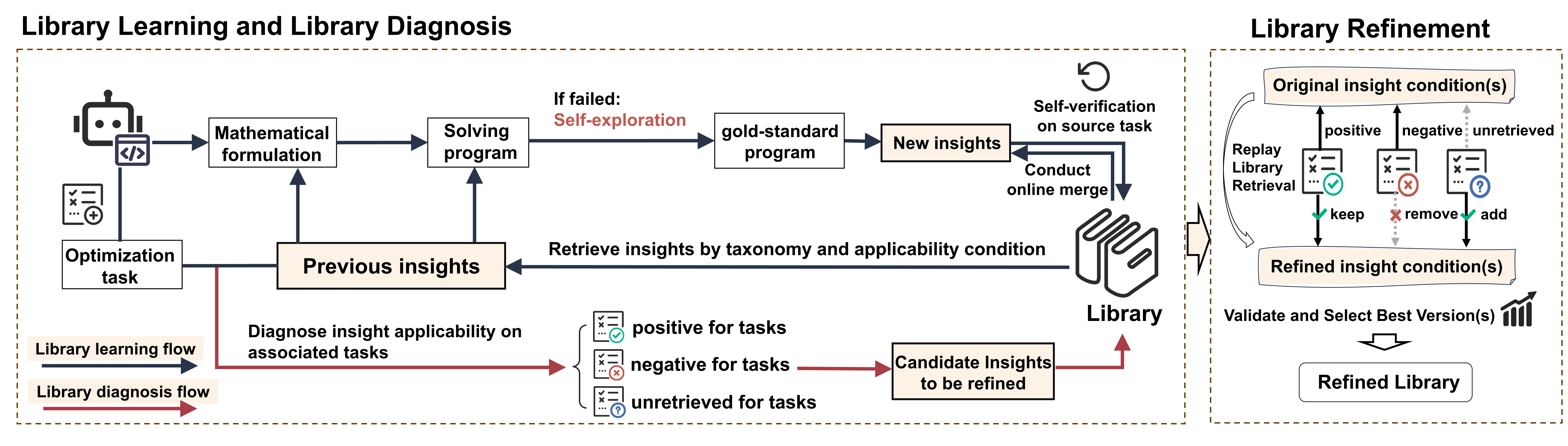}
  \caption{The overall workflow of AlphaOPT. The left panel depicts two complementary flows: library learning, which self-explores correct patterns and extracts new insights by trial and error; library diagnosis, which categorizes interactions between failed tasks and retrieved insights into positive, negative, and unretrieved cases. The right panel shows library refinement, whereby the LLM refines insight applicability, validates via retrieval replay, and reintegrates them into the library.}
  \Description{The figure illustrates how the insight library is learned, diagnosed, and refined through interaction with optimization tasks. On the left, the library learning and diagnosis process is shown. Given an optimization task, the system retrieves relevant prior insights based on a taxonomy and their applicability conditions. These insights guide the construction of a mathematical formulation, which is then translated into a solving program. If the program succeeds, the task provides positive evidence for the retrieved insights. If it fails, the system performs self-exploration by comparing the failed solution with a gold-standard program and extracting new insights. All insights, both existing and newly generated, are stored in the library. At the same time, the system diagnoses insight applicability across related tasks. For each task, an insight can be observed as helpful, harmful, or not retrieved at all. Insights that consistently perform poorly or are missing when needed are marked as candidates for refinement. On the right, the library refinement process is shown. Candidate insights are replayed through library retrieval and self-verification on source tasks. Based on positive, negative, or unretrieved outcomes, their applicability conditions are updated by keeping, removing, or adding conditions. The refined versions are validated, and the best-performing versions are selected to form the updated library. Overall, the figure shows a closed loop in which solving optimization tasks both uses the library and continuously improves it by refining when each insight should apply.}
  \label{fig:workflow}
\end{figure*}

\subsection{Library Learning
\label{sec:method:library_learning}} 

This stage is designed to generate reusable insights as structured 4-tuples (\textit{Taxonomy}, \textit{Condition}, \textit{Explanation}, \textit{Example}) and organize them in a hierarchical taxonomy dictionary for efficient retrieval while minimizing redundancy in the library. Figure~\ref{fig:workflow} (Library Learning Flow) clearly illustrates the overall workflow.

\paragraph{\textbf{Insight Extraction and Representation}}
For each arriving task, the system first constructs a mathematical formulation, then generates an executable solver program and invokes the solver. When the library is non-empty, both steps are guided by retrieved insights. If the generated program does not achieve the correct optimal value, the system performs solver-guided self-exploration: It iteratively proposes executable programs, reuses prior failures as context, and receives verification from the solver. Once a program achieves its correct objective, it is treated as a proxy for the gold standard in anchor insight extraction.

Each insight is represented as a structured 4-tuple: \textit{Taxonomy}, hierarchical labels for indexing and retrieval; \textit{Condition}, an explicit description of the applicable prerequisites and triggering signals in the problem; \textit{Explanation}, the underlying principle of applying this insight; and \textit{Example}, a concrete demonstration such as corresponding mathematical expressions or code snippet. 

\paragraph{\textbf{Insight Verification.}}
To ensure the retrievability and correctness of newly generated natural-language insights and to prevent the uncontrolled accumulation of idle or invalid entries in the library, each insight undergoes two stages of strict verification before being stored:
\emph{Retrieval verification}: temporarily incorporate the insight into a copy of the library, and check whether the source task can match it via LLM-driven retrieval;
\emph{Execution verification}: reapply the new insight together with other retrieved insights to its source task to verify that the new insight can resolve the original failure and lead to optimality.

At each stage, a limited number of rewrite attempts is allowed until the insight passes verification, during which the LLM performs targeted revisions of the taxonomy, condition or the whole insight to improve its efficacy.

\paragraph{\textbf{Library Storage.}}
Experience library is maintained as a dynamically updating taxonomy dictionary organized into three main tracks: \emph{Domain Modeling} (problem-specific structures and assumptions), \emph{General Formulation} (reusable mathematical patterns), and \emph{Code Implementation} (solver-specific coding practices). Within each track, the library organizes insights using a two-level labeling scheme, where Level-1 captures a broad category and Level-2 specializes it into a more specific subcategory. The taxonomy dictionary is initialized with few-shot labels and expands online: Each new insight is either mapped to an existing category or, if no suitable label exists, prompts the LLM to propose new Level-1 or Level-2 labels. Each label is also assigned a condition, written by the LLM, that specifies when the category should be retrieved. 

To prevent uncontrolled growth of the experience library, we incorporate an online merging mechanism throughout training to prevent degradation in retrieval and execution efficiency. For each newly generated insight, the system performs redundancy checking with existing insights, merging when appropriate, followed by correctness validation on all associated tasks.

\paragraph{\textbf{Library Retrieval.}}
During solution generation, retrieved insights from the \emph{Domain Modeling} and \emph{General Formulation} tracks guide the construction of the mathematical model, while insights from the \emph{Code Implementation} track guide solver-code generation. To align a target task with relevant insights, we employ a two-step LLM-driven retrieval procedure: Quick label matching, then full applicability check. The system first scans the taxonomy dictionary to identify labels that are potentially relevant to the context of the tasks. For example, a Level-2 label such as Fixed Charge (Big-M Linking) will probably be detected when the problem description specifies that service or flow is allowed only if a facility is opened. After candidate labels are identified, the system rigorously evaluates each associated insight by examining its condition, and only the most applicable insights are retained. By employing a progressive disclosure strategy that reveals only task-relevant information when it becomes applicable, this procedure mitigates context overflow while maintaining the effectiveness of retrieval.

\subsection{Library Evolution
\label{sec:method:library_evolution}} 

While Library Learning expands the repository of insights, Library Evolution aims to transform task-specific lessons into broadly applicable knowledge. Since each insight’s applicability is defined by a condition induced by a specific task, early conditions are often too narrow (failing to trigger on relevant tasks) or too broad (causing misretrieval). Left unchecked, these misalignments lead to widespread misguidance or missed opportunities. Library Evolution counters this with a diagnosis–refinement cycle: It detects misaligned insights, aggregates evidence across tasks, and refines conditions at the end of each iteration. The refinement is guided by an aggregate metric rather than ad hoc fixes. As illustrated in Figure~\ref{fig:lib-refine}, library refinement can be understood as adjusting each insight’s condition toward the correct retrieval boundary in the problem space. 

\paragraph{\textbf{Library Diagnosis.}}
 After an initial pass over the training tasks, the system continuously diagnoses task failures and traces their interactions with library insights through iterative cycles of insight retrieval and solution generation across tasks. By comparing the formulation and program generated under the guidance of retrieved insights with the ground truth to locate discrepancies, the diagnostic LLM agent partitions the relationship between each insight $i$ and its associated tasks into three disjoint categories:
$\Pi(i)=\{\text{Positive}:S_i^+, \text{Negative}:S_i^-, \text{Unretrieved}:S_i^u\}$
where $S_i^{+}$ contains tasks for which the insight was applicable and contributed to the correct formulation and program, $S_i^{-}$ contains tasks for which it was misleading and degraded performance, and $S_i^{u}$ contains tasks for which it was not retrieved but would have been beneficial. By maintaining these partitions across iterations, the system continuously builds a performance profile for each insight. If a failed task is subsequently solved after removing a misleading (negative) insight or by injecting a previously unretrieved one, the system attributes the failure to condition misalignment rather than lack of knowledge and thus avoids redundant insight generation. Figure~\ref{fig:workflow} (Library Diagnosis Flow) illustrates the workflow for this stage.

\begin{figure}
  \centering
  \includegraphics[width=0.7\linewidth]{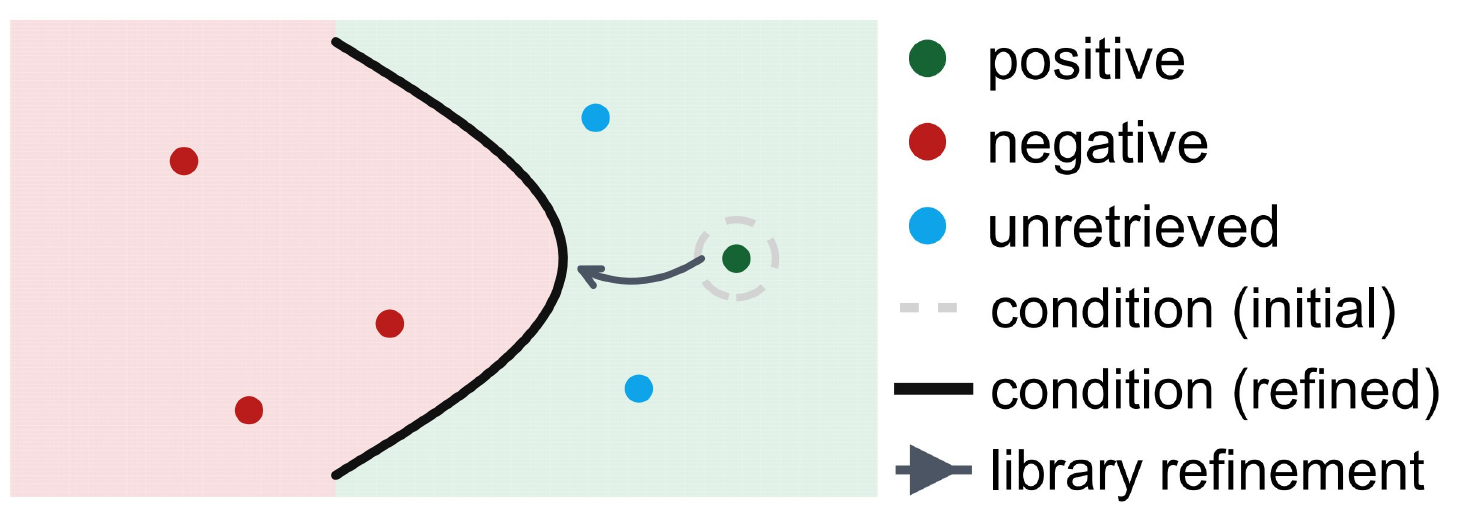}
  \caption{A locally verified \emph{initial} condition is refined into a broader \emph{applicability} boundary through LLM-driven refinement.}
  \Description{The figure provides an abstract illustration of how an insight’s applicability condition is refined over time. Each point represents a task outcome, where green points indicate tasks for which the insight is helpful, red points indicate tasks for which it is harmful, and blue points indicate tasks where the insight was not retrieved. The dashed curve shows the initial applicability condition, which is locally verified but narrow. Through library refinement, the condition boundary is adjusted, shown as the solid curve, to better separate positive cases from negative ones and to cover relevant unretrieved cases. The arrow indicates the refinement process, resulting in a broader and more accurate applicability boundary for the insight.}
  \label{fig:lib-refine}
\end{figure}

\paragraph{\textbf{Library Refinement.}}
Building on the diagnosed interaction of each insight, the LLM agent first strengthens or prunes its applicability condition. Negative tasks contribute explicit \emph{inapplicability clauses} (e.g., constraints or contexts that block use), while unretrieved tasks highlight missing applicability signals. Multiple refinement strategies are asked to be proposed (e.g., adding preconditions, introducing keyword anchors, merging overlapping triggers) and produce candidate conditions with the goal of preserving correct cases, eliminating mismatches, and recovering previously missed tasks. Then, each candidate condition replaces the original and is tested over the union $R_i = S_i^{+} \cup S_i^{-} \cup S_i^{u}$. A performance score
    {\small
    \[
    p_i = \frac{|\text{kept positives}| + |\text{corrected negatives}| + |\text{recovered unretrieved}|}{|R_i|}
    \]
    }
    quantifies improvement. Here, \emph{kept positives} are tasks that remain correctly retrieved after refinement; \emph{corrected negatives} are tasks that were misled by the insight before and are no longer retrieved; and \emph{recovered unretrieved} are tasks that become correctly retrieved after refinement. We accept refinements that increase $p_i$ and keep the one with the highest $p_i$. The workflow for this stage is illustrated in the right panel of Figure~\ref{fig:workflow}.

\begin{table*}[t]
\centering
\small
\setlength{\tabcolsep}{2.2pt}
\renewcommand{\arraystretch}{1.15}
\caption{Accuracy (\%) on \emph{Test Split} and \emph{Out-of-Distribution} datasets, along with micro- and macro-averaged accuracy across datasets. Best per column in \textbf{bold} and second best underlined (ablation variants excluded from highlighting).}
\label{tab:overall_results}
\begin{tabular}{c c c c c c c c c c c c}
\toprule
& & \multicolumn{6}{c}{\textbf{Test Split}} & \multicolumn{4}{c}{\textbf{Out-of-Distribution}} \\
\cmidrule(lr){3-8} \cmidrule(lr){9-12}

\multicolumn{2}{c}{\textbf{Method}}
& \makecell{NLP4LP\\(73)}
& \makecell{NL4OPT\\(64)}
& \makecell{IndustryOR\\(25)}
& \makecell{MAMO\\(ComplexLP)\\(34)}
& \makecell{Micro}
& \makecell{Macro}
& \makecell{LogiOR\\(92)}
& \makecell{OptiBench\\(403)}
& \makecell{Micro}
& \makecell{Macro} \\
\midrule

\multirow{3}{*}{\emph{Prompt-based}}
& Standard 
& 68.5 & 54.7 & 52.0 & 44.1 
& 57.7 & 54.8 
& 46.7 & 72.7 
& 67.9 & 59.7 \\
& OptiMus 
& 71.2 & 73.4 & 36.0 & 29.4 
& 60.2 & 52.5 
& 17.4 & 74.7 
& 64.1 & 46.0 \\
& ORThought 
& 69.9 & 75.0 & \textbf{60.0} & 41.2 
& 65.3 & 61.5 
& 44.6 & \underline{84.4} 
& \underline{77.0} & \underline{64.5} \\
\midrule

\multirow{2}{*}{\emph{Fine-tuning}}
& ORLM
& \underline{86.3} & \textbf{87.5} & 36.0 & 55.9
& \underline{75.0} & \underline{66.4}
& 19.6 & 78.2
& 67.3 & 48.9 \\
& LLMOPT
& 83.6 & 79.7 & 28.0 & 41.2
& 67.9 & 58.1
& 27.2 & 79.7
& 69.9 & 53.5 \\
\midrule

\multirow{3}{*}{\emph{Experience-learning}}
& Reflexion 
& 76.7 & 64.1 & \underline{56.0} & 47.1 
& 64.8 & 61.0 
& 43.5 & 76.9 
& 70.7 & 60.2 \\
& Expel
& 79.5 & 67.2 & 48.0 & \underline{70.4}
& 69.9 & 66.3
& \underline{47.8} & 80.4
& 74.3 & 64.1 \\
\rowcolor{lightgray}
& \textbf{AlphaOPT}
& \textbf{87.7} & \underline{81.2} & \textbf{60.0} & \textbf{76.5}
& \textbf{80.1} & \textbf{76.4}
& \textbf{56.0} & \textbf{93.5}
& \textbf{86.5} & \textbf{74.8} \\
\midrule

\multirow{4}{*}{\emph{Ablation (AlphaOPT)}}
& w/o Taxonomy 
& \makecell{84.9 \\ {\footnotesize($\downarrow 2.8$)}}
& \makecell{79.9 \\ {\footnotesize($\downarrow 1.3$)}}
& \makecell{60.0 \\ {\footnotesize($-$)}}
& \makecell{79.4 \\ {\footnotesize($\uparrow 2.9$)}}
& \makecell{79.1 \\ {\footnotesize($\downarrow 1.0$)}}
& \makecell{76.1 \\ {\footnotesize($\downarrow 0.3$)}}
& \makecell{48.9 \\ {\footnotesize($\downarrow 7.1$)}}
& \makecell{92.5 \\ {\footnotesize($\downarrow 1.0$)}}
& \makecell{84.4 \\ {\footnotesize($\downarrow 2.1$)}}
& \makecell{70.7 \\ {\footnotesize($\downarrow 4.1$)}} \\
& w/o Applicability Condition 
& \makecell{83.6 \\ {\footnotesize($\downarrow 4.1$)}}
& \makecell{81.2 \\ {\footnotesize($-$)}}
& \makecell{56.0 \\ {\footnotesize($\downarrow 4.0$)}}
& \makecell{67.6 \\ {\footnotesize($\downarrow 8.9$)}}
& \makecell{76.5 \\ {\footnotesize($\downarrow 3.6$)}}
& \makecell{72.1 \\ {\footnotesize($\downarrow 4.3$)}}
& \makecell{47.8 \\ {\footnotesize($\downarrow 8.2$)}}
& \makecell{91.8 \\ {\footnotesize($\downarrow 1.7$)}}
& \makecell{83.6 \\ {\footnotesize($\downarrow 2.9$)}}
& \makecell{69.8 \\ {\footnotesize($\downarrow 5.0$)}} \\
& w/o Refinement 
& \makecell{82.2 \\ {\footnotesize($\downarrow 4.1$)}}
& \makecell{80.7 \\ {\footnotesize($\downarrow 0.5$)}}
& \makecell{44.0 \\ {\footnotesize($\downarrow 16.0$)}}
& \makecell{61.8 \\ {\footnotesize($\downarrow 14.7$)}}
& \makecell{73.3 \\ {\footnotesize($\downarrow 6.8$)}}
& \makecell{67.2 \\ {\footnotesize($\downarrow 9.2$)}}
& \makecell{48.9 \\ {\footnotesize($\downarrow 7.1$)}}
& \makecell{92.1 \\ {\footnotesize($\downarrow 1.4$)}}
& \makecell{84.1 \\ {\footnotesize($\downarrow 2.4$)}}
& \makecell{70.5 \\ {\footnotesize($\downarrow 4.3$)}} \\

& w/o Insight Example
& \makecell{86.3 \\ {\footnotesize($\downarrow 1.4$)}}
& \makecell{78.1 \\ {\footnotesize($\downarrow 3.1$)}}
& \makecell{56.0 \\ {\footnotesize($\downarrow 4.0$)}}
& \makecell{58.8 \\ {\footnotesize($\downarrow 17.7$)}}
& \makecell{76.0 \\ {\footnotesize($\downarrow 4.1$)}}
& \makecell{69.8 \\ {\footnotesize($\downarrow 6.6$)}}
& \makecell{45.7 \\ {\footnotesize($\downarrow 10.3$)}}
& \makecell{91.0 \\ {\footnotesize($\downarrow 2.5$)}}
& \makecell{82.9 \\ {\footnotesize($\downarrow 3.6$)}}
& \makecell{68.4 \\ {\footnotesize($\downarrow 6.4$)}} \\

\bottomrule
\end{tabular}
\end{table*}

\subsection{Optimization Perspective\label{sec:method:theory}}
The framework can be viewed as an iterative solution to the optimization problem in the library space. Let $\mathcal{L}$ denote a candidate library and $\mathcal{T}$ the distribution of the optimization problems we want to solve. The objective is to maximize task success while penalizing library complexity to mitigate retrieval inefficiency and long-context degradation in LLM inference:

\[
\max_{\ell \in \mathcal{L}} \;\; \mathbb{E}_{t \sim \mathcal{T}}\!\left[ \mathrm{Success}(t \mid \ell) \right] \;-\; \lambda \,\Omega(\ell),
\]
where $\mathrm{Success}(t \mid \ell)$ indicates whether $\ell$ enables the system to produce a program that achieves the correct optimal objective for task $t$, and $\Omega(\ell)$ quantifies library complexity (e.g., number of insights or redundancy-adjusted size). Under our problem design---bounded and continuous property of $\text{Success}(\cdot)$ and $\Omega (\cdot)$, sufficient exploration under solver verification, and bounded merging---the refinement dynamics converge to a locally optimal library. In Appendix~\ref{sec:appendix:proof}, we provide a conceptual sketch that shows that convergence holds: Since refinement in the second phase strictly improves the first term, while verified merging in the first phase reduces the second term without diminishing the first, sufficient exploration combined with iterative cycles of library learning and evolution ensures convergence to a local optimum. Given the inherent ambiguity of natural language and stochasticity in LLM outputs, we present this perspective not as a strict theorem but as a principled justification for the acquisition–refinement design and redundancy-reduction operations.

\section{Experiments}
\label{headings}

Our experiments are designed to reflect the requirements that arise in real-world optimization and operations research applications. In these settings, methods are expected not only to perform well on standard benchmarks but also to transfer across domains, remain effective when limited supervision is available, improve steadily as more data become available, and offer results that can be inspected and audited. We organize our empirical evaluation around four questions: (1) How well does the method generalize across domains? (2) How effective are the insight content structure and library refinement mechanisms? (3) Does performance improve consistently as more training data become available? (4) Does the library size converge as training progresses?

\subsection{Experimental Setup}
Our experiments are conducted on a dataset of 452 problem instances, aggregated from four real-world optimization task datasets---the NLP4LP \citep{ahmaditeshnizi2024optimus}, NL4OPT \citep{NL4OPT}, IndustryOR \citep{huang2025orlm}, MAMO (ComplexLP) \citep{huang2024mamo}.  These collections span various formulation types and originate from diverse sources, which include academic papers, textbooks, and real‐world industry scenarios, and each instance pairs a natural-language problem description with its optimal solution. Detailed dataset statistics are provided in Appendix~\ref{appendix:datasets}.

We perform stratified sampling within each dataset, randomly partitioning $70\%$ for training and $30\%$ for testing. The experience library is constructed only from training tasks, which ensures that training-derived insight examples do not leak into the test set. To assess out-of-distribution generalization, we also evaluate on LogiOR \citep{yang2025orthought} and OptiBench \citep{yang2024optibench}. Because our framework derives feedback from correct solutions, it is relatively sensitive to data noise; we therefore train and evaluate on clean splits that exclude instances labeled as erroneous. Specifically, for NLP4LP, IndustryOR, and LogiOR we use the cleaned versions provided by \citet{yang2025orthought}, and for NL4OPT, MAMO (ComplexLP), and OptiBench we use the cleaned releases from \citet{astorga2025autoformulation}.

Unless otherwise specified, GPT-4o (OpenAI, 2024) with temperature 0 is used as the backbone. We use the success rate as the primary evaluation metric, following the evaluation protocol of \citet{yang2025orthought}, whereby a task is considered successful if the LLM-generated optimal value closely aligns with the provided ground-truth solution. We also report micro- and macro-averaged accuracy for both test-split and out-of-distribution datasets.\footnote{Given evaluation datasets $\mathcal{D}=\{D_1,\dots,D_K\}$ with $|D_k|$ instances and per-dataset accuracy $\mathrm{Acc}(D_k)$, the micro-averaged accuracy $\mathrm{Acc}_{\text{micro}}=\sum_{k} |D_k|\,\mathrm{Acc}(D_k) \big/ \sum_{k} |D_k|$ weights datasets by size, whereas the macro-averaged accuracy $\mathrm{Acc}_{\text{macro}}=\frac{1}{K}\sum_{k} \mathrm{Acc}(D_k)$ weights all datasets equally.}

\textbf{Baselines.} We evaluate against three families of baselines. (i) \textbf{Prompt Designing}: directly generate the mathematical model from manually crafted prompts, which include OptiMUS \citep{ahmaditeshnizi2024optimus}, and ORThought \citep{yang2025orthought}. (ii) \textbf{Fine-tuning}: adapt model parameters through fine-tuning, including ORLM \citep{huang2025orlm}, built on LLaMa3-8B, and LLMOPT \citep{JiangShu2025llmopt}, built on Qwen2.5-14B. To remove train--test overlap concerns, we additionally retrain LLMOPT on \textbf{AlphaOPT}'s training split and report its retrained results in our comparison. (iii) \textbf{Experience Learning}: Accumulate reusable knowledge from past task-solving attempts to improve decision quality across tasks, including two representative works Reflexion \citep{shinn2023reflexion} and ExpeL \citep{zhao2023expel}. For ExpeL, we adapt the original framework to operations research tasks by making minor changes to task-related prompts and evaluation criteria.

\subsection{Overall Performance}

Table~\ref{tab:overall_results} summarizes the overall performance across test and out-of-distribution benchmarks. \textbf{AlphaOPT} consistently outperforms most of the baselines on the test splits, with only ORLM remaining higher on the individual NL4OPT, and it also achieves the best performance on the most challenging datasets IndustryOR and LogiOR, exceeding the second-best method by an average of 6\%.

While AlphaOPT adopts a stronger backbone than the fine-tuned baselines, these gains do not stem from it alone. Fine-tuning and AlphaOPT represent complementary mechanisms rather than the same capability at different scales: fine-tuning encodes domain patterns into the model parameters, whereas AlphaOPT stores solver-verified knowledge in a reusable external library applied at inference time. We transfer AlphaOPT's learned library to LLMOPT (a fine-tuned Qwen2.5-14B) at inference time, which still yields a 6\% average gain, confirming that the external library, not the backbone, drives the improvements.
Also, fine-tuning's apparent advantages are less conclusive. Many existing benchmarks are derived from a small set of seed problems~\citep{NL4OPT,huang2024mamo}, which favors fine-tuning approaches that excel at pattern memorization and inflates their in-distribution scores. Moreover, fine-tuning methods require detailed annotations of mathematical formulations and code to reach strong performance, whereas AlphaOPT attains competitive results using only answer labels.\footnote{Since AlphaOPT relies on self-exploration rather than gold programs, an incorrect program could in principle pass by reaching the correct optimal value. Auditing all self-explored tasks against reference solutions, we found such spurious passes negligible (3 cases, 0.8\%, with only minor, numerically inconsequential variable-type imprecision); stronger safeguards (e.g., multi-instance verification, LLM-based consistency checking) are left to future work.} This advantage is particularly significant in OR practice, where high-quality, expert-annotated datasets are generally scarce.
Additionally, compared with baseline experience-learning methods, our approach achieves superior performance as methods such as ExpeL tend to learn generic insights that lack the detailed modeling and coding guidance needed to prevent common errors. The full list of ExpeL-learned insights is provided in Appendix~\ref{appendix:expel_insights}.

Robustness to distribution shift is where our framework's advantage is most pronounced: on LogiOR~\citep{yang2025orthought} and OptiBench~\citep{yang2024optibench}, two benchmarks not seen during training, fine-tuned models degrade sharply---ORLM and the retrained LLMOPT reach only 19.6\% and 27.2\% on LogiOR, respectively. In contrast, AlphaOPT attains 56.0\% on LogiOR and 93.5\% on OptiBench, indicating that its experience library captures transferable modeling principles rather than dataset-specific solution patterns, leading to stronger robustness under distribution shift. Beyond optimization formulation, AlphaOPT further generalizes to OR algorithm design and general-purpose code generation, indicating that it is a broadly effective framework for knowledge-intensive tasks, unlike OR-specialized fine-tuned models. See Appendix~\ref{appendix:cross_domain} for experimental details.

\subsection{Ablation Study}
We ablate the core components of \textbf{AlphaOPT} to isolate their individual contributions. \textbf{w/o Taxonomy} removes the hierarchical taxonomy, so insights are retrieved without label-based organization; \textbf{w/o Applicability Condition} drops the explicit applicability conditions, matching insights by taxonomy label alone.
We also evaluate \textbf{w/o Refinement}, in which insights are stored after extraction without the Library Evolution stage that adjusts their applicability conditions, and \textbf{w/o Insight Example}, in which insights omit their concrete mathematical or code-level demonstrations.

The results show that both retrieval components matter: dropping the applicability condition costs 3.6 and 4.3 points on the micro- and macro-averaged accuracy of the test split, and dropping the taxonomy lowers the out-of-distribution macro average by 4.1 points. Most tellingly, replacing the two-step retrieval with embedding-based similarity causes the largest drop, reducing the test-split micro average by 4.9 points and the out-of-distribution macro average by 5.4 points, as similarity matches surface text rather than logical applicability, misretrieving insights for structurally different problems. This confirms that AlphaOPT's gains stem from its structured retrieval, not the insight content alone.

Removing refinement leads to consistent performance drops across both test-split and out-of-distribution datasets. On the test split, micro- and macro-averaged accuracy decrease by 6.8 and 9.2 points respectively, with the largest degradation on structurally complex datasets such as IndustryOR and MAMO (ComplexLP). Out-of-distribution performance also declines, indicating that unrefined applicability conditions are insufficient for reliable cross-task reuse. Removing insight examples yields a smaller but systematic drop of 4.1 and 6.6 points in test-split micro and macro accuracy, and 17.7 on the formulation-heavy MAMO (ComplexLP), confirming that worked demonstrations help translate principles into executable formulations.

\subsection{Continual Growth and Library Convergence}
We test whether \textbf{AlphaOPT} improves as more data become available by training on incremental subsets of 100, 200, and 300 items. As shown in Table~\ref{tab:continual_learning_results}, out-of-distribution accuracy (LogiOR, OptiBench) steadily improves with data size, without any model-parameter updates.

This also empirically validates the analysis in Section~\ref{sec:method:theory}, under which AlphaOPT converges to a locally optimal library. As shown in Figure \ref{fig:lib_converge}, cumulative training accuracy rises rapidly early and then saturates, aside from a minor fluctuation during the mid-training stage attributable to the stochasticity of LLM generation, while the library size slows and stabilizes. This stems from online insight merging, solver-based verification, and refinement, which progressively saturate the number of distinct, verifiable insights that can be added.

Together, continual improvement alongside a converging library shows that AlphaOPT scales without unbounded accumulation of redundant insights. This convergence also bounds cost at scale: once the library stabilizes, additional tasks are largely absorbed by existing insights rather than triggering new learning, so token cost grows sub-linearly with dataset size; Appendix~\ref{appendix:cost} details the cost profile and two readily deployable cost-reduction strategies for large-scale industrial deployment.

\begin{table}[t]
\centering
\caption{AlphaOPT steadily improves in both Micro and Macro averages with increasing training size.}
\label{tab:continual_learning_results}

\setlength{\tabcolsep}{6pt}
\renewcommand{\arraystretch}{1.15}

\begin{tabular}{c c c c}
\toprule
Training Size & MicroAvg (\%) & MacroAvg (\%) \\
\midrule
100 & 83.24 & 65.80  \\
200 & 85.09 & 69.22  \\
300 & 85.21 & 72.12  \\
\bottomrule
\end{tabular}
\end{table}

\begin{figure}[th!]
  \centering
  \includegraphics[width=1\linewidth]{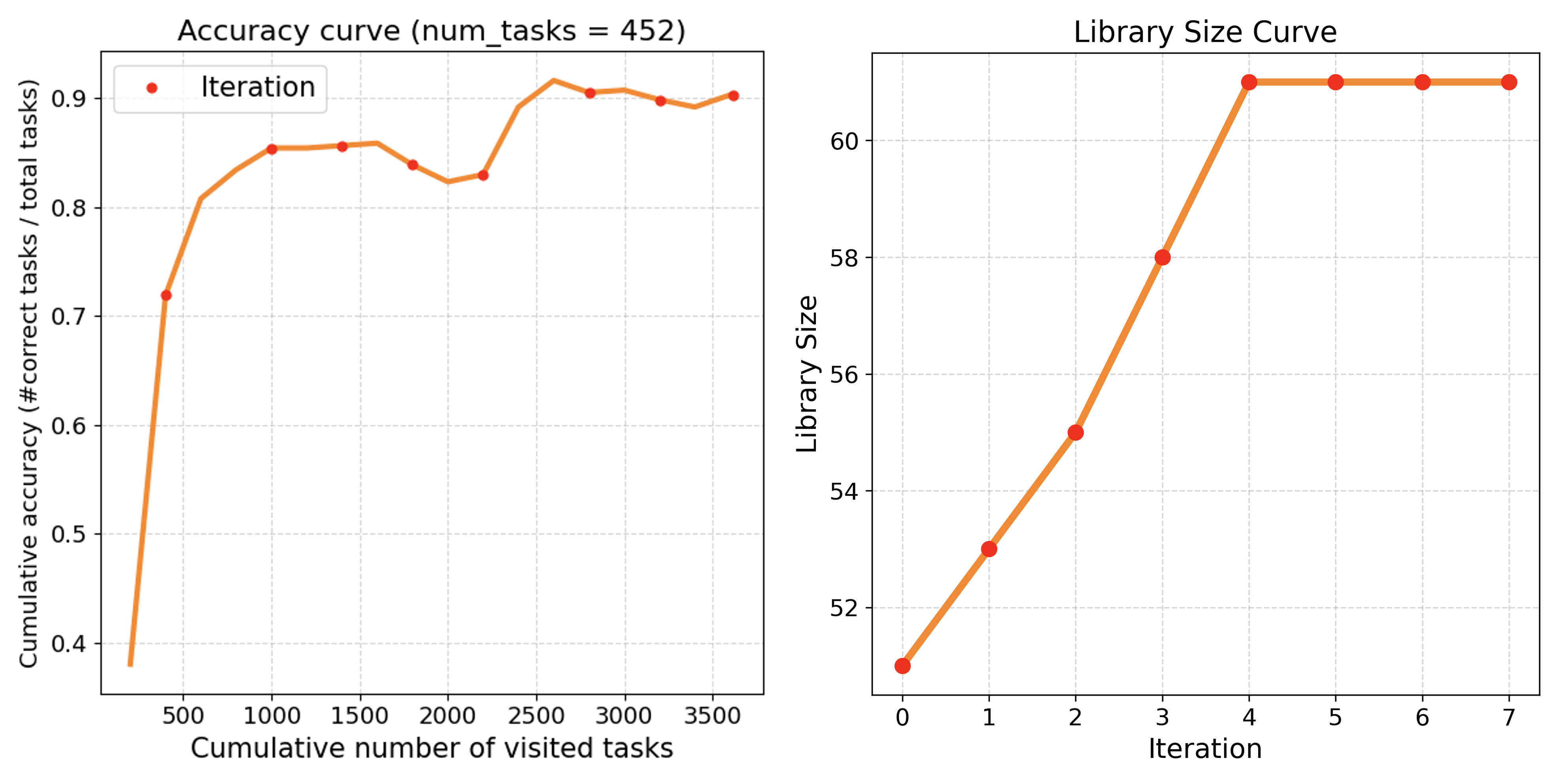}
  \caption{Training dynamics of the experience library. As training progresses, accuracy converges while library growth stabilizes.}
  \Description{The figure shows the training dynamics of the experience library across iterations. The left plot reports cumulative accuracy as more tasks are visited. Accuracy increases rapidly at the beginning, then improves more gradually, and eventually stabilizes as additional tasks are processed. The right plot shows the size of the experience library over training iterations. The library grows steadily in the early iterations and then remains mostly the same after four iterations. Together, the two plots illustrate that performance converges while the library size stabilizes.}
  \label{fig:lib_converge}
\end{figure}

\section{Library Analysis}

\subsection{Library Taxonomy and Knowledge}

\begin{figure}[t]
  \centering
\includegraphics[width=1\linewidth]{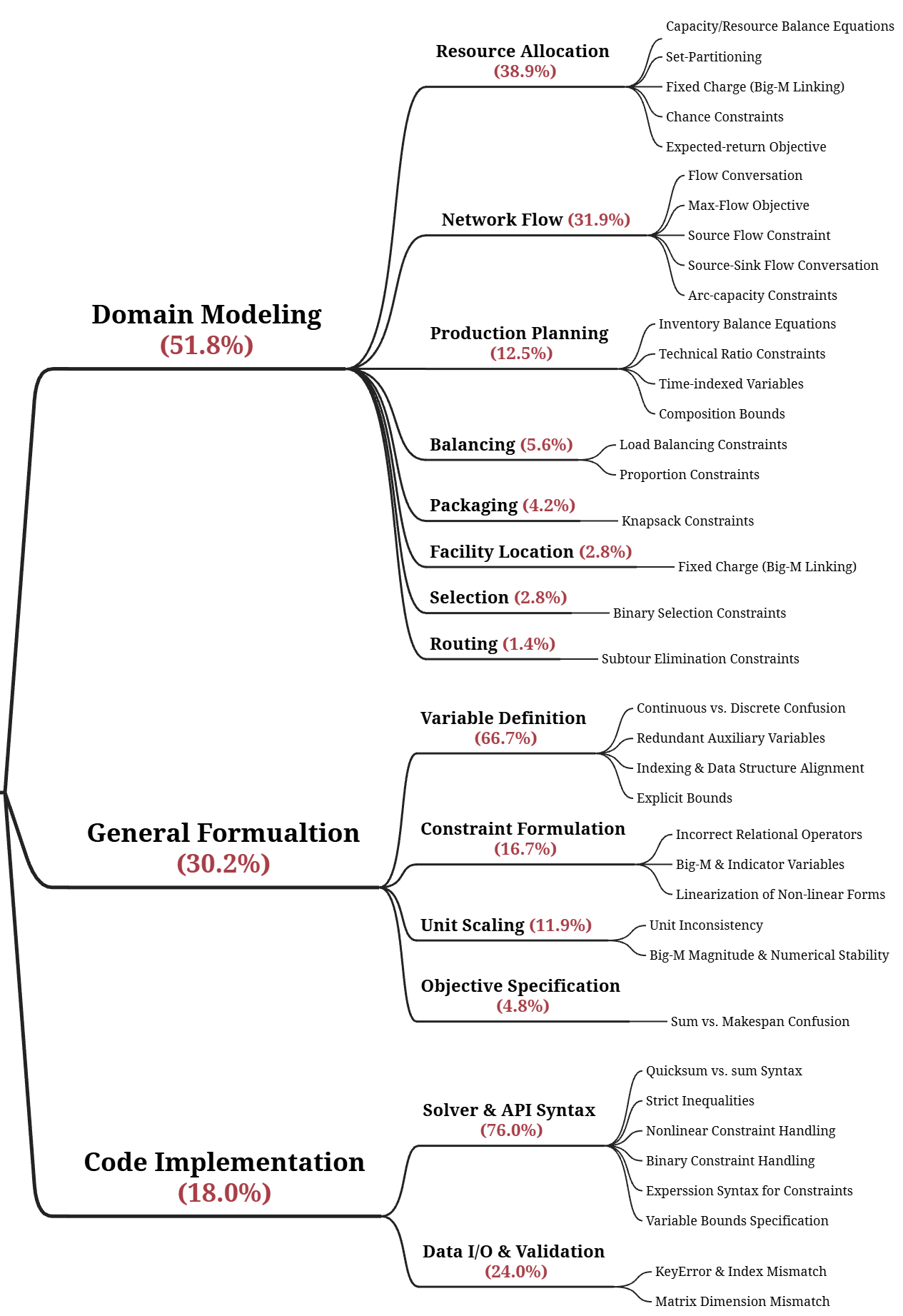}
  \caption{The experience library taxonomy consists of three main tracks, 14 level-1 labels, and 38 level-2 labels. Percentages shown in red indicate their proportions.} 
  \Description{The figure shows the structure of the learned experience library taxonomy. The taxonomy is organized into three main tracks: Domain Modeling, General Formulation, and Code Implementation. Each track is further divided into level-1 categories, which are then broken down into more specific level-2 concepts. Domain Modeling covers common problem structures such as resource allocation, network flow, production planning, balancing, packaging, facility location, selection, and routing. For each category, representative modeling elements are listed, such as capacity and flow balance equations, fixed-charge modeling, inventory balance, and subtour elimination. General Formulation focuses on how optimization models are expressed, including variable definition, constraint formulation, unit scaling, and objective specification, with examples of typical pitfalls such as continuous versus discrete confusion, Big-M misuse, and incorrect objective definitions. Code Implementation captures issues that arise when translating formulations into solver code, including solver and API syntax, constraint expression syntax, and data input, output, and validation errors. Percentages shown in red indicate the proportion of experiences assigned to each track and to each level-1 category, reflecting how frequently different types of insights appear in the learned library.}
  \label{fig:library_analysis}
\end{figure}

To fully interpret the experience library, we visualize its hierarchical taxonomy and the distribution of insights (as shown in Figure~\ref{fig:library_analysis}). Detailed level-2 label distributions and the full taxonomy specification are provided in Appendix~\ref{appendix:lvl2_insights} and Appendix~\ref{appendix:taxonomy}.

Overall, LLMs tend to make mistakes more frequently on mathematical modeling than in code implementation, especially when problem structures are complex or modeling conventions are implicit. We summarize the challenges as threefold:

\paragraph{\textbf{Structural Coupling and Balance.}}
In the Domain Modeling track, LLMs exhibit two primary difficulties: (i) capturing structural coupling across variables or stages, and (ii) maintaining global balance constraints over system resources or quantities. These challenges are particularly pronounced in problem classes such as \emph{Resource Allocation} (38.9\%), \emph{Network Flow} (31.9\%), and \emph{Production Planning} (12.5\%). For example, in production planning they often fail to couple machine-level decisions with inter-temporal inventory balance, instead constraining total processing time via aggregated variables and violating structural dependencies.

\paragraph{\textbf{From Intuition to Rigorous Formulation.}}
In the General Formulation track, LLMs are prone to pitfalls when translating intuitive or contextual descriptions into mathematically rigorous and solver-consistent formulations, with mistakes concentrated in \emph{Variable Definition} (66.7\%), \emph{Constraint Formulation} (16.7\%), and \emph{Unit Scaling} (11.9\%). These errors include unclear variable domains, confusion between continuous and discrete variables, misuse of relational operators, improper handling of Big-M or indicator variables, and unit inconsistencies. For example, \emph{at least}/\emph{at most} constraints (e.g., minimum production or budget limits) are often modeled as equalities or strict inequalities, yielding overly restrictive formulations.

\paragraph{\textbf{Bridging Models and Executable Code.}}
In the Code Implementation track, although such errors are less frequent, LLMs still exhibit gaps when translating symbolic mathematical models into executable solver code. Among these errors, 76\% stem from \emph{Solver \& API Syntax} (e.g., improper handling of nonlinear constraints, confusion between \texttt{quicksum} and \texttt{sum}), while the remaining 24\% arise from \emph{Data I/O \& Validation} issues (e.g., \texttt{KeyError}, index mismatches, and matrix dimension errors).

\subsection{Examples of Library Refinement\label{refined_examples}}

Representative refined insights are selected from the library to illustrate how their applicability conditions are adjusted based on the associated tasks.
\begin{figure}[th!]
  \centering
  \includegraphics[width=1\linewidth]{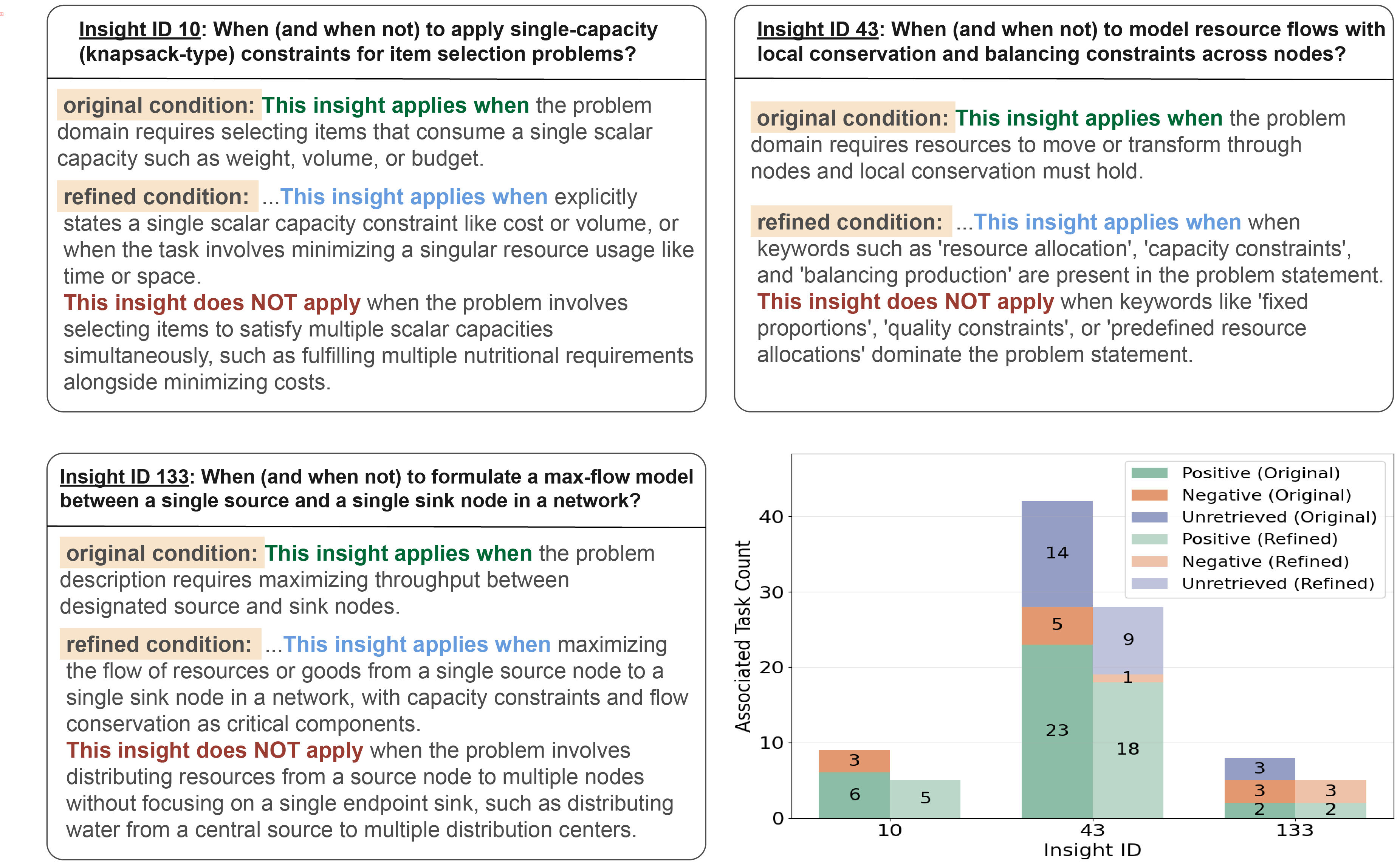}
  \caption{Examples of insight applicability conditions before and after refinement, along with the corresponding changes in task associations.}
  \Description{The figure shows concrete examples of insight applicability conditions before and after refinement, together with how task associations change as a result. The figure presents three representative insights, each identified by an ID, along with their original applicability conditions and the refined conditions learned through library refinement. The examples cover different modeling patterns, including single-capacity item selection, resource flow modeling with conservation constraints, and max-flow formulations between a source and a sink. For each insight, the refined condition is more specific about the problem features that must be present and explicitly states when the insight should not be applied. At the bottom right, the figure shows a bar chart summarizing how tasks are associated with each insight before and after refinement. Task outcomes are grouped into positive, negative, and unretrieved cases for both the original and refined conditions. After refinement, positive associations increase while negative and unretrieved associations decrease, indicating that the refined conditions better match the tasks where each insight is appropriate.}
  \label{fig:refined_examples}
\end{figure}

As shown in the Figure~\ref{fig:refined_examples}, \emph{Insight ID 10} targets single–scalar-capacity selection; the refinement broadens coverage by adding equivalent formulations such as minimizing a single resource (time/space), while explicitly excluding multi-capacity settings. \emph{Insight ID 43}  captures flow/conservation and production balancing across nodes; refinement adds lexical anchors (e.g., capacity constraints, balancing production) and excludes statements dominated by fixed proportions, quality constraints, or predefined allocations, and gating suppresses spurious retrievals whereby the structure is not true flow/conservation. \emph{Insight ID 133} focuses on max-flow between a single source and a single sink; refinement tightens the structural requirement and rules out source-to-many distribution tasks. 

Across the cases, four LLM-driven refinement strategies recur: (i) generalization via equivalent phrasings; (ii) lexical anchoring with positive keywords; (iii) explicit exclusions to reduce misalignment; and (iv) structural qualifiers to prevent misuse. As shown in Figure~\ref{fig:refined_examples}, these semantic refinements reduce both negative and unretrieved cases while preserving most positive cases.

\section{Conclusion and Discussion}
\label{conclusion}

This paper addresses the limitations of previous methods by presenting \textbf{AlphaOPT}, a novel self-improving library learning framework for formulating optimization programs. Its core innovations lie in organizing experience as a self-updating hierarchical taxonomy of structured insights, generating and updating these insights under solver verification, and continually refining their applicability conditions from cross-task evidence. Learning from answer labels alone, it achieves much stronger out-of-distribution generalization than fine-tuning–based methods while providing interpretable, auditable knowledge that exposes LLMs’ characteristic failure patterns across domain modeling, formulation, and solver syntax. Finally, the framework generalizes beyond optimization formulation to broader domains such as OR algorithm design and general-purpose code generation, suggesting its promise as a general paradigm for knowledge-intensive reasoning.

Looking ahead, we highlight three promising directions. First, reasoning-oriented test-time scaling is especially promising for OR, whose results are inherently verifiable. Second, scaling to large-scale, real-world industrial problems beyond the toy examples that dominate current benchmarks, and moving beyond correctness toward more efficient algorithms, are both crucial for practical deployment, and our self-improving library offers a promising path toward them. Third, as LLM-based agentic coding frameworks become an increasingly effective tool for complex reasoning and software-engineering tasks, the condition-aware insights learned here could be elevated into higher-level, composable agent skills that abstract away from surface narratives and transfer at the workflow level.

\begin{acks}
This research is supported by the National Research Foundation (NRF), Prime Minister's Office, Singapore, under its Campus for Research Excellence and Technological Enterprise (CREATE) programme. The Mens, Manus, and Machina (M3S) is an interdisciplinary research group (IRG) of the Singapore-MIT Alliance for Research and Technology (SMART) centre.
This research is also supported by the National Research Foundation, Singapore under its AI Singapore AI Research Fundamental Research Collaborative (US-NSF Researcher Call) (AISG Award No: AISG3-RP-2025-036-USNSF).
\end{acks}

\newpage
\FloatBarrier
\bibliographystyle{ACM-Reference-Format}
\bibliography{reference}

\appendix

\FloatBarrier   

\section{Proof of the Library Convergence\label{sec:appendix:proof}}
Recall the optimization problem in the library training phase
\[
F(\ell) \;=\; \mathbb{E}_{t \sim \mathcal{T}_{\mathrm{train}}}\!\big[\, r(t \mid \ell)\,\big] \;-\; \lambda \, \Omega(\ell),
\]
where \(r(t,\ell)\) is a bounded reward function that implements the role of the original \(\mathrm{Success}(t\mid\ell)\) (i.e., it measures the matching quality between optimization problem \(t\) and library \(\ell\)), and \(\Omega(\ell)\) is a bounded complexity penalty.  

According to the problem setting, the iterative refinement algorithm satisfies:
\begin{enumerate}
    \item Monotone update: At iteration $k$, from $\ell_k$, the algorithm considers a set of admissible refinements $R(\ell_k) \subseteq \mathcal{L}$. Each accepted iteration consists of one of two types of operations:
    \begin{itemize}
      \item \emph{Merge step:} decreases $\Omega(\ell)$ while leaving $r(t,\ell)$ non-decrease for all relevant tasks;
      \item \emph{Exploration step:} improves $r(t,\ell)$ for some tasks without increasing $\Omega(\ell)$.
    \end{itemize}
    Therefore, every accepted refinement strictly increases $F(\ell)$; otherwise, the algorithm keeps $\ell_{k+1} = \ell_k$.
    \item Sufficient exploration: Any improving neighbor $\tilde{\ell} \in R(\ell_k)$ (i.e.\ one with a strictly larger objective) will eventually be discovered and executed. Empirically, this is achieved through iterative prompt optimization with LLMs.
    \item Boundedness: $r(t,\ell)$ and $\Omega(\ell)$ are bounded; hence $F(\ell)$ is bounded above and below. 
\end{enumerate}
The following theorem establishes that, under the assumption that the training and testing distributions are identical, the refinement procedure yields libraries that are locally optimal for the testing objective.
\begin{theorem}
Assume $\mathcal{T}_{\mathrm{train}} = \mathcal{T}_{\mathrm{test}}$. 
If the library space $\mathcal{L}$ is finite, the algorithm terminates in finitely many steps at a library $\ell^*$ that is a local maximizer for the testing objective. 
Moreover, the algorithm cannot terminate at a saddle point.
\label{thm:finite}
\end{theorem}

\begin{proof}
Every accepted merge or exploration step strictly increases $F(\ell)$; otherwise, the library remains unchanged. Since $F$ is bounded above, the sequence $\{F(\ell_k)\}$ is monotone non-decreasing and bounded, and hence convergent to some limit $F^*$. Furthermore, since $\mathcal{L}$ is finite, define
\[
\delta \;=\; \min \big\{ F(\tilde{\ell}) - F(\ell) : \tilde{\ell} \in R(\ell),\; F(\tilde{\ell}) > F(\ell) \big\}.
\]

Finiteness guarantees $\delta>0$, so only finitely many strict improvements are possible. The algorithm halts at some $\ell^*$. By sufficient exploration, no improving neighbor of $\ell^*$ exists. Therefore, $\ell^*$ is a local maximizer for both training and testing objectives. Saddle points are excluded.

Since the training and testing distributions coincide, the training objective equals the testing objective; thus any local optimality statement directly applies to testing.
\end{proof}
Although the assumption of a finite library is reasonable, we also provide a proof for the case of an infinite library for completeness and rigor.
\begin{theorem}[Infinite compact library case]
Assume $\mathcal{T}_{\mathrm{train}} = \mathcal{T}_{\mathrm{test}}$. 
If the library space $\mathcal{L}$ is compact (closed and bounded) and $F$ is continuous, then the sequence $\{F(\ell_k)\}$ converges, and any subsequential limit point $\ell^\infty$ is a local maximizer for the testing objective. 
Saddle points are excluded for all such limit points.
\label{thm:infinite}
\end{theorem}

\begin{proof}
Each accepted step strictly increases $F(\ell)$, so $\{F(\ell_k)\}$ is monotone non-decreasing. Since $F$ is bounded above, $\{F(\ell_k)\}$ converges to some $F^*$. By compactness of $\mathcal{L}$, there exists a convergent subsequence $\ell_{k_j}\to\ell^\infty$. The continuity of $F$ ensures $F(\ell_{k_j}) \to F(\ell^\infty)=F^*$. Suppose $\ell^\infty$ had a neighbor $\tilde{\ell}\in R(\ell^\infty)$ with $F(\tilde{\ell})>F(\ell^\infty)$. Then sufficient exploration would eventually yield $F(\ell_{k}) > F^*$, which is a contradiction. Therefore, $\ell^\infty$ is a local maximizer. Saddle points are excluded. Since the training and testing distributions coincide, the training objective equals the testing objective; thus any local optimality statement directly applies to testing.
\end{proof}

Theorems~\ref{thm:finite} and~\ref{thm:infinite} together guarantee that, when the training and testing distributions coincide, the refinement algorithm converges to locally optimal solutions for the testing phase.

\section{Experimental Details}

\subsection{Datasets\label{appendix:datasets}}

We list the optimization problem datasets used in this work in Table~\ref{table:datasets}; each pairs a natural-language problem description with its optimal solution, and the Size column reports each dataset's size after cleaning. The gold-standard programs for the training datasets NLP4LP and IndustryOR are obtained from \citet{yang2025orthought}, and the cleaned releases of \citet{astorga2025autoformulation} are available from their \href{https://github.com/LLM4OR/LLM4OR}{GitHub repository}.

\begin{table}[t]
  \centering
  \caption{Statistics of the optimization problem datasets}
  \label{table:datasets}
  \resizebox{\columnwidth}{!}{%
      \begin{tabular}{l l l l}
      \toprule
      \textbf{Dataset}   & \textbf{Size}   & \textbf{Formulation Type(s)}   & \textbf{Completion}   \\
      \midrule
      NL4OPT \citep{NL4OPT}  & 289    & LP   & solution   \\
      NLP4LP \citep{ahmaditeshnizi2024optimus} & 242   & LP, MILP, MINLP & solution, program \\
      MAMO (complex LP) \citep{huang2024mamo}  & 111   & LP   & solution   \\
      IndustryOR \citep{huang2025orlm}   & 82   & LP, IP, MILP, NLP, others   &   solution, program   \\
      OptiBench  \citep{yang2024optibench}   & 403   & LP, MILP, MINLP   & solution   \\
      LogiOR \citep{yang2025orthought} & 92 & LP, IP, MIP, NLP & solution\\ 
      \bottomrule
    \end{tabular}
    }
  \vspace{0.5ex}

  \footnotesize
  Abbreviations:  
  LP – Linear Programming;  
  IP - Integer Programming;
  NLP – Nonlinear Programming;  
  MI – Mixed-Integer;  
  others - Quadratic Programming, Dynamic\&Stochastic Programming, etc.
\end{table}

\subsection{Insights Learned Through Baseline Experience Learning Methods\label{appendix:expel_insights}}

To better understand the limitations of prior experience-learning methods in optimization formulation, we examine all insights extracted by ExpeL when adapted to operations research tasks. The learned insights are dominated by generic, high-level recommendations, listed below:
\begin{itemize}
    \item Ensure that decision variables are defined with the correct type (integer, binary, or continuous) based on the problem requirements, especially for inherently discrete quantities.
    \item Double-check that all constraints are correctly implemented and aligned with the problem conditions, including the correct mathematical representation and direction of inequalities.
    \item Verify that the objective function accurately reflects the intended optimization goal.
    \item Always validate the model output against expected results to ensure feasibility and correctness. 
    \item Regularly review the problem statement to ensure correct interpretation of variable types and constraints. 
    \item Ensure that decision variables have appropriate bounds to reflect realistic limits and prevent infeasible or unbounded solutions. 
    \item Incorporate logical constraints to model conditional relationships between decision variables. 
    \item Test the model under different scenarios to assess robustness and adaptability. 
    \item Re-validate variable types when transitioning between continuous and discrete representations. 
    \item Check the model structure for potential infeasibilities or unboundedness before optimization. 
    \item Consider edge cases to improve robustness across problem conditions. 
    \item Ensure that the objective and constraints scale appropriately with problem size or parameter changes. 
    \item Examine the impact of variable bounds on feasibility and correctness. 
\end{itemize}

These insights emphasize general modeling hygiene rather than concrete modeling or coding practices. They do not encode explicit applicability conditions, domain-specific structural patterns (e.g., flow conservation, inventory balance, fixed-charge activation), or solver-level implementation details. Consequently, while such insights are broadly reusable, they provide limited actionable guidance for correcting specific formulation or implementation errors in complex optimization problems.

\subsection{Generalization Beyond Optimization Formulation\label{appendix:cross_domain}}
Although our main experiments target optimization formulation, \textbf{AlphaOPT}'s learn--diagnose--refine loop is domain-agnostic. To test whether it generalizes beyond toy-scale optimization formulation tasks to broader problem styles, we apply the \emph{same} framework with a lightweight GPT-4o-mini backbone to two benchmarks outside our main scope: \emph{CO-Bench}~\citep{sun2025cobenchbenchmarkinglanguagemodel}, a solver-free Python algorithm-synthesis suite spanning 31 real-world combinatorial-optimization problem types ($\sim$3{,}000 instances, 18-train/13-test), and \emph{HumanEval}~\citep{chen2021evaluatinglargelanguagemodels}, general-purpose code generation (164 tasks, 114-train/50-test).

\begin{table}[!h]
\centering
\small

\caption{Cross-domain transfer of AlphaOPT.}
\label{tab:cross_domain}
\setlength{\tabcolsep}{5pt}
\renewcommand{\arraystretch}{1.1}
\begin{tabular}{lccc}
\toprule
Benchmark & Train & Test & Test $\Delta$ \\
\midrule
CO-Bench  & 69.1 $\rightarrow$ 81.1 & 54.2 $\rightarrow$ 61.2 & +7.0 \\
HumanEval & 35.1 $\rightarrow$ 77.2 & 24.7 $\rightarrow$ 52.0 & +27.3 \\
\bottomrule
\end{tabular}
\end{table}

\begin{table*}[!h]
\centering
\caption{Computation cost and runtime statistics across training and testing stages.}
\label{tab:cost_runtime}
\begin{tabular}{cccccccc}
\hline
Stage & Phase & Data Size & Prompt Tokens & Completion Tokens & Cost per instance & Runtime & Steps \\
\hline
Training & Library Learning & 452 & $\sim$7.7M & $\sim$1.3M & $\sim$\$0.07 & $\sim$40.5 min & 1 \\
Training & Library Diagnosis       & 452 & $\sim$11.5M & $\sim$2.3M & $\sim$\$0.11  & $\sim$48.6 min  & 3-5 \\
Training & Library Refinement      & 452 & $\sim$32.9M & $\sim$1.9M & $\sim$\$0.22 & $\sim$33.7 min  & 3-5 \\
Testing  &  & 691 & $\sim$13.8M & $\sim$1.5M & $\sim$\$0.07  & $\sim$25.5 min  & 3 \\
\hline
\end{tabular}
\end{table*}

On both held-out splits AlphaOPT yields clear gains (+7.0 on CO-Bench and +27.3 on HumanEval, more than doubling the HumanEval baseline), and the learned libraries (28 and 46 insights) organize into the same Algorithm Design and Code Implementation tracks as in OR formulation. This shows the framework transfers beyond optimization formulation to algorithm/heuristic design and general-purpose code generation, supporting its applicability to larger-scale and more diverse tasks.

\subsection{Experimental Cost and Scalability\label{appendix:cost}}

All experiments use OpenAI's GPT-4o via API, and Table~\ref{tab:cost_runtime} reports token consumption across phases, where \emph{Steps} denotes the number of execution rounds per module and the Diagnosis--Refinement loop typically converges within five rounds. Refinement dominates the prompt-token cost, since each insight is refined together with all of its associated tasks as context, whereas testing---only solution generation and insight reuse---is substantially cheaper. The average cost stays around \$0.1 per instance, and this relatively high cost is incurred only once when building the library and amortized over all subsequent reuse, giving the system strong economic scalability.

Crucially, this cost does not grow linearly with dataset size: as the library converges (Figure~\ref{fig:lib_converge}), most new tasks are absorbed by existing insights rather than triggering new learning--diagnosis--refinement, so the marginal cost per task decreases. Two readily deployable optimizations further improve large-scale efficiency: (1)~selecting representative training instances, which lowers training cost without meaningfully diminishing experiential-learning gains; and (2)~gating refinement by a failure threshold, so computation concentrates on the small fraction of still-failing tasks rather than re-refining already-reliable insights.

\subsection{Insight Distribution\label{appendix:lvl2_insights}}

To further understand the detailed content of the library insights, we analyze the distributions of insights under the level-2 taxonomy labels, as well as the contribution differences among training datasets (as depicted in Figure~\ref{fig:lvl2_insights}).

\begin{figure}[!ht]
  \centering
  \begin{minipage}[t]{\linewidth}
    \centering
    \includegraphics[width=1\linewidth]{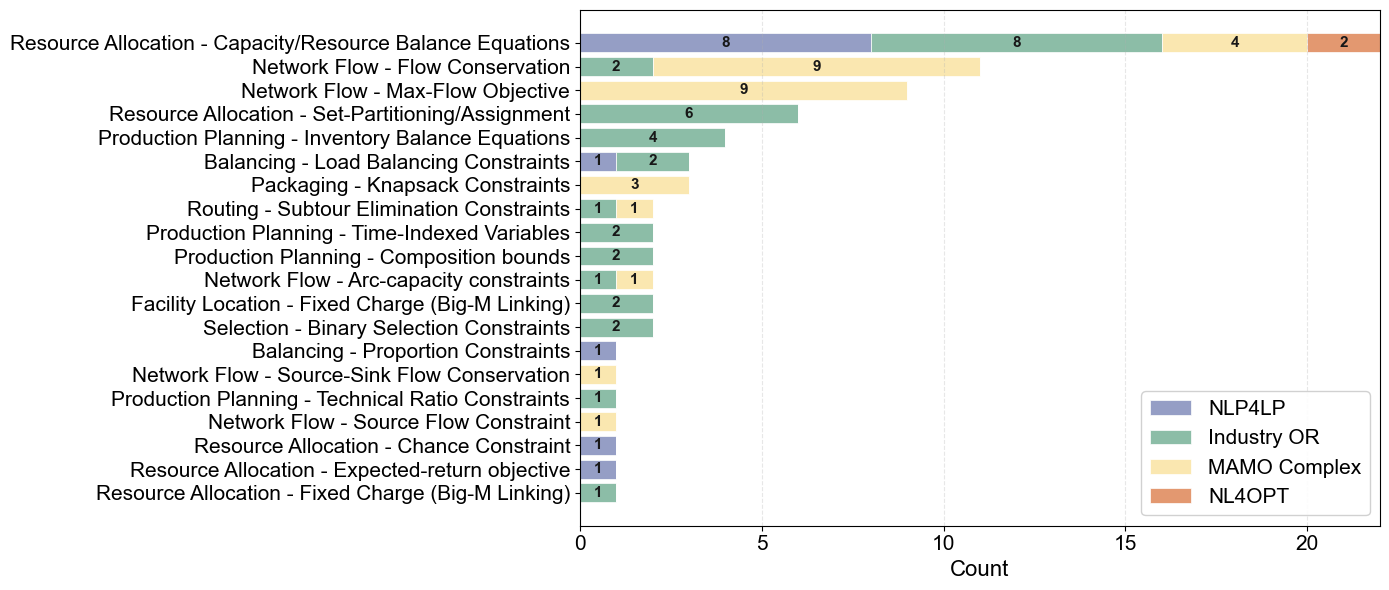}
    \caption*{(a) Insight distribution of Domain Modeling track}
  \end{minipage}
  
  \begin{minipage}[t]{\linewidth}
    \centering
    \includegraphics[width=1\linewidth]{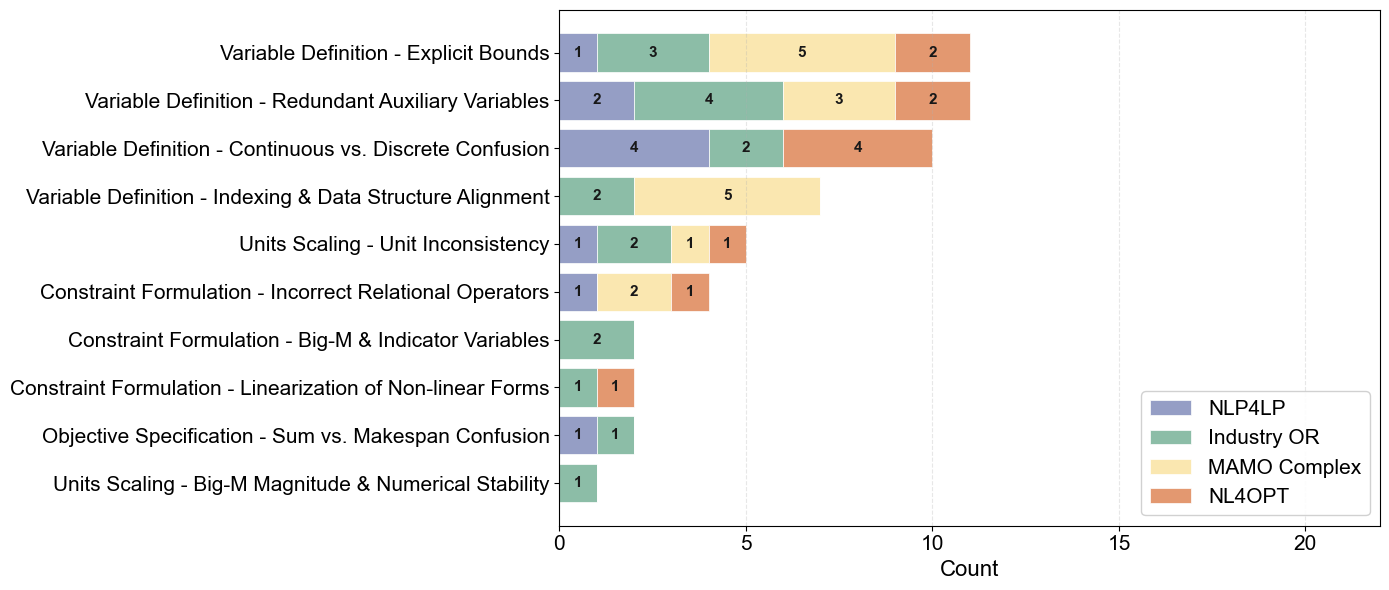}
    \caption*{(b) Insight distribution of General Formulation track}
  \end{minipage}
  
    \begin{minipage}[t]{\linewidth}
    \centering
    \includegraphics[width=1\linewidth]{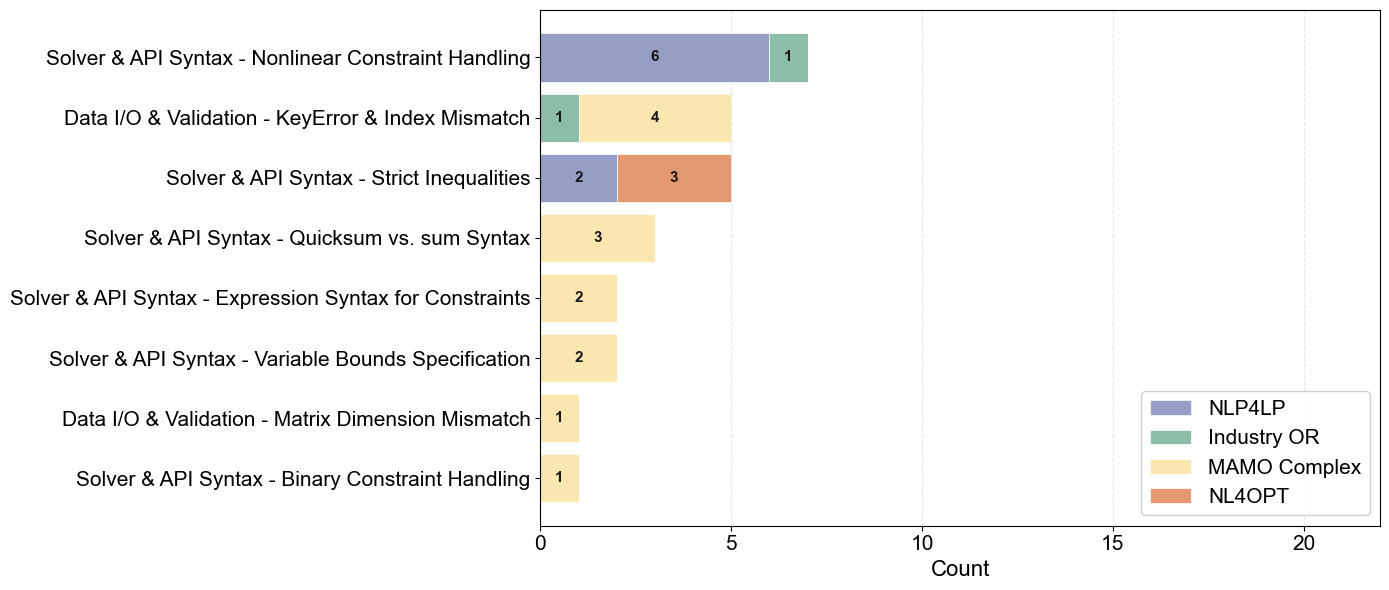}
    \caption*{(c) Insight distribution of Code Implementation track}
  \end{minipage}
  \caption{Insight distributions across three different tracks, showing the contributions of source tasks in the four training datasets to the generation of library insights.}
  \Description{The figure shows how insights in the experience library are distributed across the three main tracks and how different source datasets contribute to each track. Panel (a) focuses on the Domain Modeling track. Each horizontal bar corresponds to a specific domain-level modeling concept, such as capacity and resource balance equations, flow conservation, max-flow objectives, or inventory balance. The bars are segmented by dataset, indicating how many insights of each type originate from NLP4LP, Industry OR, MAMO Complex, and NL4OPT tasks. This panel highlights that domain modeling insights are concentrated around a small number of common structures, especially resource allocation and network flow patterns, with contributions coming from multiple datasets. Panel (b) shows the distribution for the General Formulation track. The listed concepts include variable definition issues, constraint formulation errors, unit scaling problems, and objective specification confusions. Compared to domain modeling, these insights are more evenly spread across formulation-level pitfalls, such as continuous versus discrete variables, redundant auxiliary variables, and Big-M related issues. Contributions again vary across datasets, reflecting differences in how formulation errors arise in different task sources. Panel (c) presents the Code Implementation track. The insights here relate to solver and API usage, constraint syntax, variable bounds, and data input or validation errors. The distribution shows fewer categories but clear concentration on solver syntax and implementation mistakes, with notable contributions from datasets that emphasize executable optimization code. Overall, the figure illustrates that different datasets contribute complementary types of insights and that the learned library captures a diverse range of issues spanning domain modeling, mathematical formulation, and code-level implementation.}
  \label{fig:lvl2_insights}
\end{figure}

In the \emph{Domain Modeling} track, insights under \emph{Resource Allocation – Capacity/Resource Balance Equations} account for the largest proportion, with contributions from all four datasets, among which NLP4LP and Industry OR contribute the most. Insights tagged with this label are widely applicable to optimization problems that require ensuring that resource usage does not exceed available capacity. This demonstrates how to establish constraints that maintain balance between resource consumption and availability.

In the \emph{General Formulation} track, \emph{Variable Definition – Explicit Bounds}, \emph{Redundant Auxiliary Variable}, and \emph{Continuous vs. Discrete Confusion} are the most frequent, with diverse sources, which indicates that these are common issues across multiple datasets. This reflects the fact that foundational concepts in variable definition are the most error-prone in optimization modeling---particularly in balancing variable types, value ranges, and modeling simplifications---while LLMs tend to produce redundant formulations or type misuse due to overlooking structural or physical consistency.

In the \emph{Code Implementation} track, the number of insights is the smallest, with \emph{Solver \& API Syntax – Nonlinear Constraint Handling} and \emph{Data I/O \& Validation – KeyError \& Index Mismatch} accounting for the highest proportions, which are mainly contributed by the MAMO (Complex LP) and NLP4LP datasets. These two types of insights reveal critical vulnerabilities in the implementation stage—solver syntax compatibility and data accessibility---and represent the dual pillars required for bridging mathematical modeling and executable code.

From the perspective of dataset contribution, Industry OR is dominant in \emph{Domain Modeling}; MAMO (Complex LP) and Industry OR lead in \emph{General Formulation}; and both MAMO (Complex LP) and NLP4LP contribute the most to \emph{Code Implementation}. NL4OPT has relatively lower overall participation but focuses on formulation- and solver-related details. Considering dataset size, although the Industry OR and MAMO (Complex LP) datasets used for library learning are only about half the size of the other datasets, they still contribute a large number of insights, which indicates that these datasets contain denser structural modeling challenges and more diverse error patterns; this enables the LLM to accumulate more experiential knowledge across multiple dimensions.

\subsection{Success and Failure Case Study\label{appendix:case_study}}
To assess the retrieved insights, we compare task solving with and without library retrieval across all evaluation datasets (the test splits of NL4OPT, NLP4LP, IndustryOR, and MAMO (ComplexLP), and the out-of-distribution sets LogiOR and OptiBench), excluding tasks that succeed in both settings, and classify each insight's outcome as \emph{success} (it matches a task and prevents an error that would otherwise occur), \emph{failure} (it mismatches a task and introduces a new error), or \emph{invalid} (it matches correctly but fails to fix the original mistake). As shown in Figure~\ref{fig:case_studies}, about half of the retrieved insights help solve new tasks while only a small portion introduce errors, with outcomes varying across level-1 taxonomy labels.

\begin{figure}[!ht]
  \centering
  \begin{minipage}[t]{\linewidth}
    \centering
    \includegraphics[width=0.6\linewidth]{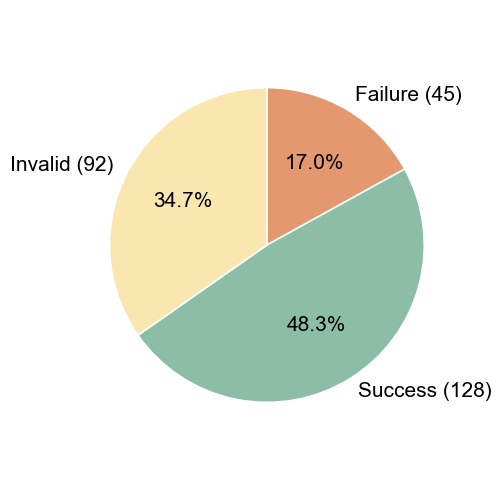}
    \caption*{(a) Overall effectiveness of retrieved library insights}
  \end{minipage}
  \begin{minipage}[t]{\linewidth}
    \centering
    \includegraphics[width=1\linewidth]{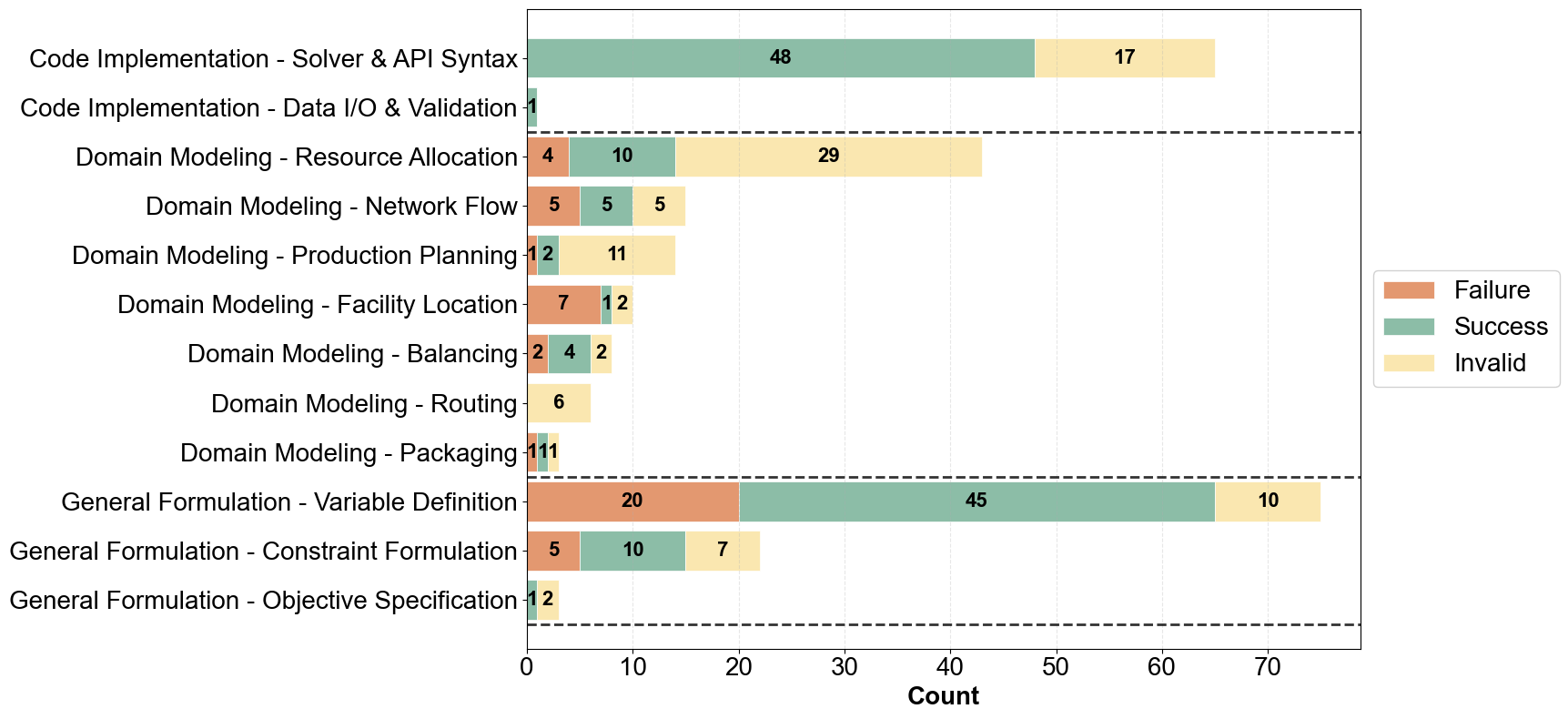}
    \caption*{(b) Effectiveness distribution across Level-1 insight categories}
  \end{minipage}
  \caption{Effectiveness of retrieved insights. (a) Overall distribution of insight outcomes. (b) Proportions of successful, failed, and invalid insights under level-1 taxonomy labels.}
  \Description{The figure evaluates how effective the retrieved library insights are when applied to new tasks. Panel (a) shows the overall outcome distribution for all retrieved insights. Each insight is classified as a success, failure, or invalid. Successful cases are those where the insight matches the task and prevents an error that would otherwise occur. Failed cases are those where the insight mismatches the task and introduces a new error. Invalid cases are those where the insight is applicable but does not help correct the original mistake. The pie chart reports both counts and proportions, showing that nearly half of the retrieved insights lead to successful corrections, while the rest are split between invalid and failed outcomes. Panel (b) breaks down these outcomes by level-1 insight categories. Each horizontal bar corresponds to a category such as domain modeling, general formulation, or code implementation, and is further divided into success, failure, and invalid portions. This view highlights differences in effectiveness across categories. For example, variable definition and solver or API syntax insights account for a large number of successful cases, while some domain modeling categories show higher proportions of invalid outcomes. Overall, the figure shows that retrieved insights are often helpful, with some variation between categories.}
  \label{fig:case_studies}
\end{figure}

\paragraph{1) Insights with high success rates}
These insights are concentrated in \emph{Code Implementation – Solver \& API Syntax} and \emph{General Formulation – Variable Definition}. These insights typically target code implementation or fundamental modeling errors with clear structures, which enable the LLM to follow their guidance stably and correctly. For instance, insights tagged with the \emph{Strict Inequality} label under \emph{Solver \& API Syntax} highlight the fact that solvers (e.g., Gurobi) do not support strict inequalities and should instead be reformulated as non-strict forms ($\leq -1$ or $\geq +1$); insights tagged with the \emph{Explicit Bounds} label under \emph{Variable Definition} emphasize that decision variables should be assigned explicit upper and lower bounds to ensure feasibility and improve solver efficiency.

\paragraph{2) Insights with high failure rates}
These insights are mainly found in \emph{Domain Modeling – Facility Location}. They often involve structural constraints and logical triggers that are easily misinterpreted or overgeneralized by the LLM. For example, insights tagged with the \emph{Fixed Charge (Big-M Linking)} label exhibit a failure rate of 70\%. Although the principle of using Big-M constraints to model on/off logic is correct, its blind application in problems without conditional activation can lead to redundant or overlapping constraints and unnecessary feasible-region reduction, which ultimately causes solver failure. Also, while insights under \emph{Variable Definition} generally achieve high success rates, some insights tagged with \emph{Explicit Bounds} label sometimes are overly rigid, which leads the LLM to impose unnecessary upper bounds and thereby restricts the feasible space and produces suboptimal solutions.

\paragraph{3) Insights with high invalid rates}
These insights are mainly under \emph{Domain Modeling – Resource Allocation} and \emph{Domain Modeling – Production Planning}. Although the LLM successfully retrieves the correct insights and identifies the corresponding problem types, it often fails to translate them into executable formulas or solver-level implementations. For instance, under \emph{Solver \& API Syntax}, \emph{Nonlinear Constraint Handling} advises linearization or the introduction of auxiliary variables for nonlinear objectives or constraints (e.g., ratios or divisions), yet the LLM frequently fails to fully execute these transformations (by neglecting auxiliary variables or mis-rewriting proportional constraints), which results in insights being recognized but not operationalized.

Also, certain tasks remained unsolved regardless of whether library retrieval was enabled, because their failure stemmed from factors beyond the current learned library's knowledge scope. In the out-of-distribution dataset LogiOR, tasks involved multi-level spatiotemporal logic and interacting constraints (e.g., capacity, timing, and flow balance) in problems such as routing, scheduling, and network flow. These challenges extend beyond the scope of the existing library, which primarily focuses on static, linear formulations. Although related taxonomy labels such as \emph{Resource Allocation} and \emph{Nonlinear Constraint Handling} exist, their granularity and depth are insufficient for modeling such complex logic. This demonstrates that the current system, while semantically generalizable, still lacks robust cross-structural transfer and context adaptation capabilities.

Overall, while successful insights constitute the majority, the results reveal several directions for improvement. First, the relatively high failure rates indicate that, despite the inclusion of condition refinement, retrieval precision can still be improved through enhanced semantic disambiguation and structural filtering. Second, insights with high invalid proportions suggest the need for clearer explanations and better-designed examples to improve pedagogical clarity and execution effectiveness. Finally, for out-of-distribution tasks, future efforts should focus on strengthening the LLM's ability to adapt and generalize retrieved insights to unseen, complex optimization scenarios. Moreover, expanding OR datasets based on LLM error typologies can further enhance experiential learning efficiency and generalization at comparable problem scales.

\subsection{Specification of Library Taxonomy Labels\label{appendix:taxonomy}}

The following Table ~\ref{tab:library_taxonomy} lists all library taxonomy labels and their corresponding conditions, which specify the applicability criteria of each level-2 label and clarify its precise meaning. According to the library taxonomy generation mechanism, each label condition is created by the LLM when the label is first introduced, and subsequent insights generated under the same label inherit that initial condition.

\begin{table*}[!t]
\centering
\caption{Full Specification of the Library Taxonomy}
\label{tab:library_taxonomy}
\small
\renewcommand{\arraystretch}{1.12}
\setlength{\tabcolsep}{3pt}
\begin{tabularx}{\textwidth}{l X}
\toprule
\textbf{Taxonomy Label} & \textbf{Condition} \\
\midrule

\multicolumn{2}{c}{\textbf{Domain Modeling}}\\
\addlinespace[3pt]

\multicolumn{2}{l}{\textbf{Resource Allocation}}\\
Capacity/Resource Balance Equations & Applies when the problem domain requires resources to move or transform through nodes and local conservation must hold. \\
Set-Partitioning/Assignment & Applies when the problem description requires each item or task to be exclusively assigned to exactly one choice among many. \\
Fixed Charge (Big-M Linking) & Applies when the problem description requires a facility, option, or mode to be activated by a binary choice. \\
Chance Constraints & Applies when the problem description sets a limit on the average chance of an adverse outcome across options or scenarios (e.g., stake, volume, or weight). \\
Expected-return Objective & Applies when the problem description calls for maximizing the average/expected payout or return across options given their win/lose probabilities. \\
\addlinespace[2pt]

\multicolumn{2}{l}{\textbf{Network Flow}}\\
Flow Conservation & Applies when the problem description involves quantities traversing a directed network and nodal balance must be maintained. \\
Max-Flow Objective & Applies when the problem description requires maximizing throughput between designated source and sink nodes. \\
Source Flow Constraint & Applies when the problem description designates a source node that distributes resources through a network to sinks and requires explicit conservation at the source. \\
Source-Sink Flow Conservation & Applies when the problem description specifies a source and a sink and requires routing/transferring flow between them with explicit balance at those terminal nodes. \\
Arc-Capacity Constraints & Applies when the problem domain contains edges with maximum throughput or capacity limits. \\
\addlinespace[2pt]

\multicolumn{2}{l}{\textbf{Production Planning}}\\
Inventory Balance Equations & Applies when the problem description involves materials or products that carry over between periods and must satisfy the stock-flow balance. \\
Technical Ratio Constraints & Applies when the problem description specifies minimum/maximum production ratios or recipe proportions between products or stages. \\
Time-Indexed Variables & Applies when the problem domain requires discrete time modeling to capture capacities, setups, or carry-over decisions. \\
Composition Bounds & Applies when the problem description specifies multiple products sharing limited resources (e.g., machine hours or labor) that require explicit per-resource capacity limits. \\
\addlinespace[2pt]

\multicolumn{2}{l}{\textbf{Balancing}}\\
Load Balancing Constraints & Applies when the problem description requires fairness or controls maximum imbalance across parallel resources. \\
Proportion Constraints & Applies when the problem description limits the maximum or minimum proportion of a resource, flow, or activity relative to the total. \\
\addlinespace[2pt]

\multicolumn{2}{l}{\textbf{Packaging}}\\
Knapsack Constraints & Applies when the problem domain requires selecting items that consume a single scalar capacity such as weight, volume, or budget. \\
\addlinespace[2pt]

\multicolumn{2}{l}{\textbf{Facility Location}}\\
Fixed Charge (Big-M Linking) & Applies when the problem description specifies that service or flow is allowed only if a facility is opened. \\
\addlinespace[2pt]

\multicolumn{2}{l}{\textbf{Selection}}\\
Binary Selection Constraints & Applies when the problem domain requires choosing a subset under count, budget, or compatibility limits. \\
\addlinespace[2pt]

\multicolumn{2}{l}{\textbf{Routing}}\\
Subtour Elimination Constraints & Applies when the problem description allows decision variables to form disconnected cycles that must be eliminated. \\

\bottomrule
\end{tabularx}
\end{table*}

\begin{table*}[!t]
\centering
\label{tab:taxonomy_library}
\small
\renewcommand{\arraystretch}{1.12}
\setlength{\tabcolsep}{3pt}
\begin{tabularx}{\textwidth}{l X}
\toprule
\textbf{Taxonomy Label} & \textbf{Condition} \\
\midrule

\multicolumn{2}{c}{\textbf{General Formulation}}\\
\addlinespace[3pt]

\multicolumn{2}{l}{\textbf{Variable Definition}}\\
Continuous vs. Discrete Confusion & Applies when decision quantities represent indivisible counts or choices versus divisible amounts such as flows. \\
Explicit Bounds & Applies when the problem description provides natural physical, economic, or logical limits that can tightly bound decision variables. \\
Indexing \& Data Structure Alignment & Applies when variables are indexed over sets or dictionaries that must align with the keys of the provided data. \\
Redundant Auxiliary Variables & Applies when auxiliary variables merely re-express existing linear combinations without adding modeling value. \\
\addlinespace[2pt]

\multicolumn{2}{l}{\textbf{Constraint Formulation}}\\
Incorrect Relational Operators & Applies when natural-language statements such as “at most” or “at least” must be translated into algebraic inequalities. \\
Linearization of Non-linear Forms & Applies when nonlinear relations among variables reduce tractability or solver performance. \\
Big-M \& Indicator Variables & Applies when constraints depend on logical on/off conditions controlled by binary variables. \\
\addlinespace[2pt]

\multicolumn{2}{l}{\textbf{Objective Specification}}\\
Sum vs. Makespan Confusion & Applies when multiple resources or activities can run in parallel and the objective is ambiguous between total completion time and makespan. \\
\addlinespace[2pt]

\multicolumn{2}{l}{\textbf{Units Scaling}}\\
Unit Inconsistency & Applies when input data come from different unit systems or incompatible measurement scales. \\
Big-M Magnitude \& Numerical Stability & Applies when the problem description uses Big-M to model on/off or conditional constraints and realistic bounds can be derived to calibrate M. \\
\midrule

\multicolumn{2}{c}{\textbf{Code Implementation}}\\
\addlinespace[3pt]

\multicolumn{2}{l}{\textbf{Solver \& API Syntax}}\\
Quicksum vs. sum Syntax & Applies when the mathematical model contains linear expressions aggregated over large index sets that should be constructed using solver-native summation operators. \\
Strict Inequalities & Applies when the mathematical model contains strict inequality relations between variables that cannot be directly handled by LP/MIP solvers. \\
Nonlinear Constraint Handling & Applies when the problem description introduces nonlinear relationships (e.g., proportions or multiplicative effects) that must be enforced in an LP/MIP model. \\
Binary Constraint Handling & Applies when the problem description involves yes/no (open/close, select/not-select) decisions that require variables restricted to $\{0,1\}$, without adding extra [0,1] constraints. \\
Expression Syntax for Constraints & Applies when the problem description specifies equality/inequality relations (e.g., balances, conservation, on/off logic) that should be encoded directly as solver expressions. \\
Variable Bounds Specification & Applies when the problem description requires the change of a variable from one value to another. \\
\addlinespace[2pt]

\multicolumn{2}{l}{\textbf{Data I/O \& Validation}}\\
KeyError \& Index Mismatch & Applies when the mathematical model contains indexed variables or parameters that are accessed with indices not present in the corresponding data structures. \\
Missing Data Defaults & Applies when the mathematical model contains optional parameters whose values may be absent in the dataset and require default assignments to preserve model validity. \\

\bottomrule
\end{tabularx}
\end{table*}

\FloatBarrier
\onecolumn
\section{Prompts For LLM Modules\label{appendix:prompts}}

\promptboxfromfile{Apply self-exploration on finding gold-standard program}{prompts/Self_Explore.txt}

\promptboxfromfile{Generate library insights}{prompts/Generate_Insights.txt}

\promptboxfromfile{Retrieve insights by matching taxonomy}{prompts/Retri_Label.txt}

\promptboxfromfile{Retrieve insights by condition applicability}{prompts/Retri_Condition.txt}

\promptboxfromfile{Diagnose issues for failed program}{prompts/Diagnose_Issues.txt}

\promptboxfromfile{Diagnose positive and negative tasks of an insight}{prompts/Diagnose_Pos_Neg.txt}

\promptboxfromfile{Diagnose unretrieved tasks of an insight}{prompts/Diagnose_Unretri.txt}

\promptboxfromfile{Refine Insight Conditions}{prompts/Ins_Refinement.txt}

\balance
\end{document}